%% file: main.tex
\documentclass{article} %
\usepackage[table]{xcolor}     %
\usepackage{iclr2025_conference,times}

\input{math_commands.tex}

\usepackage{hyperref}
\usepackage{url}

\usepackage{amsthm}
\usepackage{bbm}
\usepackage{algorithm}
\usepackage{algpseudocode}
\usepackage{graphicx}
\usepackage{mathtools}
\usepackage{subcaption}
\usepackage{booktabs}
\usepackage{enumitem}

\usepackage[section]{placeins}

\definecolor{gold}{rgb}{1.0, 0.84, 0.0}
\definecolor{silver}{rgb}{0.75, 0.75, 0.75}
\definecolor{bronze}{rgb}{0.8, 0.5, 0.2}

\usepackage[toc,page,header]{appendix}

\newcommand{\Input}{\item[\textbf{Input:}]}
\newcommand{\Output}{\item[\textbf{Output:}]}

\theoremstyle{plain}
\newtheorem{theorem}{Theorem}[section]
\newtheorem{lemma}[theorem]{Lemma}

\newtheorem{fact}[theorem]{Fact}
\newtheorem{definition}[theorem]{Definition}

\usepackage{thm-restate}

\input{bib_macros}

\title{Provably Accurate Shapley Value Estimation via Leverage Score Sampling}

\author{Christopher Musco \\ 
  New York University \\ 
  \texttt{cmusco@nyu.edu} 
  \And
  R. Teal Witter \\ 
  New York University \\ 
  \texttt{rtealwitter@nyu.edu} 
}

\iclrfinalcopy %
\begin{document}

\maketitle

\begin{abstract}
Originally introduced in game theory, Shapley values have emerged as a central tool in explainable machine learning, where they are used to attribute model predictions to specific input features. However, computing Shapley values exactly is expensive: for a general model with $n$ features, $O(2^n)$ model evaluations are necessary. To address this issue, approximation algorithms are widely used. One of the most popular is the Kernel SHAP algorithm, which is model agnostic and remarkably effective in practice. At the time of our original submission, no strong non-asymptotic complexity guarantees were known for Kernel SHAP. We address this issue by introducing \emph{Leverage SHAP}, a lightweight modification of Kernel SHAP that provides provably accurate Shapley value estimates with just $O(n\log n)$ model evaluations. Our approach takes advantage of a connection between Shapley value estimation and agnostic active learning by employing \emph{leverage score sampling}, a powerful regression tool. Beyond theoretical guarantees, we find that Leverage SHAP achieves an approximately $50\%$ reduction in error compared to the highly optimized implementation of Kernel SHAP in the widely used SHAP library [Lundberg \& Lee, 2017].
\end{abstract}

\section{Introduction}
While AI is increasingly deployed in high-stakes domains like education, healthcare, finance, and law, increasingly complicated models often make predictions or decisions in an opaque and uninterpretable way.
In high-stakes domains, transparency in a model is crucial for building trust.
Moreover, for researchers and developers, understanding model behavior is important for identifying areas of improvement and applying appropriate safeguards.
To address these challenges, Shapley values have emerged as a powerful game-theoretic approach for interpreting even opaque models \citep{shapley1951notes,vstrumbelj2014explaining,DattaSenZick:2016,lundberg2017unified}. These values can be used to effectively quantify the contribution of each input feature to a model's output, offering at least a partial, principled explanation for why a model made a certain prediction. 

Concretely, Shapley values originate from game-theory as a method for determining fair `payouts' for a cooperative game involving $n$ players. The goal is to assign higher payouts to players who contributed more to the cooperative effort. Shapley values quantify the contribution of a player by measuring how its addition to a set of other players changes the value of the game. Formally, let the \emph{value function} $v: 2^{[n]} \to \mathbb{R}$ be a function defined on sets $S \subseteq [n]$.
The Shapley value for player $i$ is:
\begin{align}
    \phi_i &= \frac1{n} \sum_{S \subseteq [n] \setminus \{i\}}
    \frac{v(S \cup \{i\}) - v(S)}{\binom{n-1}{|S|}}.
    \label{eq:shap_standard}
\end{align}
The denominator weights the marginal contribution of player $i$ to set $S$ by the number of sets of size $|S|$, so that the marginal contribution to sets of each size are equally considered. With this weighting, Shapley values are known to be the unique values that satisfy four desirable game-theoretic properties: Null Player, Symmetry, Additivity, and Efficiency \citep{shapley1951notes}. For further details on Shapley values and their theoretical motivation, we refer the reader to \cite{Molnar:2024}.

A popular way of using Shapley values for explainable AI is to attribute predictions made by a model $f: \mathbb{R}^n \to \mathbb{R}$ on a given input $\mathbf{x} \in \mathbb{R}^n$ compared to a baseline input $\mathbf{y} \in \mathbb{R}^n$ \citep{lundberg2017unified}.
The players are the features and $v(S)$ is the prediction of the model when using the features in $S$; i.e., $v(S) = f(\mathbf{x}^S)$ where $\mathbf{x}^S_i = x_i $ for $i \in S$ and $\mathbf{x}^S_i = y_i$ otherwise.\footnote{There are multiple ways to define the baseline $\mathbf{y}$. The simplest is to consider a fixed vector \citep{lundberg2017unified}. Other approaches define $\mathbf{y}$ as a random vector that is drawn from a data distribution \citep{lundberg2018consistent,janzing2020feature} and take $v(S) = \E[f(\mathbf{x}^S)]$, where the expectation is taken over the random choice of $\mathbf{y}$. The focus of our work is estimating Shapley values once the values function $v$ is fixed, so our methods and theoretical analysis are agnostic to the specific approach used.}
When $v$ is defined in this way, Shapley values measure how each feature value in the input contributes to the prediction.

Shapley values also find other applications in machine learning, including in feature and data selection. For feature selection, the value function $v(S)$ is taken to be the model loss when 
using the features in $S$ \citep{marcilio2020explanations,fryer2021shapley}.
For data selection, $v(S)$ is taken to be the loss when using the data observations in $S$ \citep{jia2019towards,ghorbani2019data}.

\subsection{Efficient Shapley Value Computation}
\label{sec:efficient_shap_vals}
Naively, computing all $n$ Shapley values according to \Eqref{eq:shap_standard} requires $O(2^n)$ evaluations of $v$ (each of which involves the evaluation of a learned model) and $O(2^n)$ time. This cost can be reduced in certain special cases, e.g. when computing feature attributions for linear models or decision trees 
\citep{lundberg2018consistent,campbell2022exact,amoukou2022accurate,chen2018shapley}.

More often, when $v$ is based on an arbitrary model, like a deep neural network, the exponential cost in $n$ is avoided by turning to approximation algorithms for estimating Shapley values, including sampling, permutation sampling, and Kernel SHAP \citep{strumbelj2010efficient,lundberg2017unified,mitchell2022sampling}. The Kernel SHAP method is especially popular, as it performs well in practice for a variety of models, requiring just a small number of black-box evaluations of $v$ to obtain accurate estimates to $\phi_1, \ldots, \phi_n$. The method is a corner-stone of the ubiquitous SHAP library for explainable AI based on Shapley values \citep{lundberg2017unified}. 

Kernel SHAP is based on an elegant connection between Shapley values and  least squares regression \citep{charnes1988extremal}. 
Specifically, let $[n]$ denote $\{1, \ldots, n\}$, $\emptyset$ denote the empty set, and $\mathbf{1}$ denote an all $1$'s vector of length $n$. The Shapley values $\boldsymbol{\phi} = [\phi_1, \ldots, \phi_n] \in \R^{n}$ are known to satisfy:
\begin{align}\label{eq:linear_regression_connection}
   \boldsymbol{\phi} =  
   \argmin_{\mathbf{x}: \langle \mathbf{x}, \mathbf{1} \rangle = v([n]) - v(\emptyset)} \| \mathbf{Zx - y} \|_2, 
\end{align}
where $\mathbf{Z} \in \mathbb{R}^{(2^n-2) \times n}$  is a specific structured matrix whose rows correspond to sets $S \subseteq [n]$ with $0 < |S| < n$, and $\mathbf{y} \in \mathbb{R}^{2^n-2}$ is vector whose entries correspond to values of $v(S)$.
(We precisely define $\mathbf{Z}$ and $\mathbf{y}$ in Section \ref{sec:background}.)

Since solving the regression problem in \Eqref{eq:linear_regression_connection} directly would require evaluating $v(S)$ for all $2^n-2$ subsets represented in $\mathbf{y}$, 
 Kernel SHAP solves the problem approximately via \emph{subsampling}. Concretely, for a given number of samples $m$ and a discrete probability distribution $\mathbf{p} \in [0,1]^{2^n-2}$ over rows in $\mathbf{Z}$, consider a sampling matrix $\mathbf{S} \in \mathbb{R}^{m \times (2^n-2)}$, where each row of $\mathbf{S}$ is $0$ except for a single entry $1/\sqrt{p_j}$ in the $j^\text{th}$ entry with probability $p_j$.
The Kernel SHAP estimate is given by
\begin{align}\label{eq:subsampled_solution}
    \tilde{\boldsymbol{\phi}} = \argmin_{\mathbf{x}: \langle \mathbf{x}, \mathbf{1} \rangle = v([n]) - v(\emptyset)} \| \mathbf{SZx - Sy} \|_2.
\end{align}
Importantly, computing this estimate only requires at most $m$ evaluations of the value function $v$, since $\mathbf{Sy}$ can be constructed from observing at most $m$ entries in $\mathbf{y}$. Value function evaluations typically dominate the computational cost of actually solving the regression problem in \Eqref{eq:subsampled_solution}, so ideally $m$ is chosen as small as possible. 
In an effort to reduce sample complexity, Kernel SHAP does not sample rows uniformly. Instead, the row corresponding to subset $S$ is chosen with probability proportional to: 
\begin{align}
\label{eq:kernel_shap_weights}
   w(|S|) = \left(\binom{n}{|S|} |S| (n-|S|)\right)^{-1}.
\end{align}
The specific motivation for this distribution is discussed further in Section \ref{sec:background}, but the choice is intuitive: the method is more likely to sample rows corresponding to subsets whose size is close to $0$ or $n$, which aligns with the fact that these subsets more heavily impact the Shapley values (\Eqref{eq:shap_standard}). Ultimately, however, the choice of $w(|S|)$ is heuristic.

In practice, the Kernel SHAP algorithm is implemented with further optimizations \citep{lundberg2017unified,covert2020improving,jethani2021fastshap}.
\begin{itemize}[leftmargin=*]
    \item  \textbf{Paired Sampling:} If a row corresponding to set $S$ is sampled, the row corresponding to its complement $\bar{S} \vcentcolon = [n] \setminus S$ is also sampled. This \textit{paired} sampling strategy intuitively balances samples so each player $i$ is involved in the exact same number of samples. Paired sampling substantially improves Kernel SHAP \citep{covert2020improving}.
    \item \textbf{Sampling without Replacement:} When $n$ is small, sampling with replacement can lead to inaccurate solutions even as the number of samples $m$ approaches and surpasses $2^n$, even though $2^n$ evaluations of $v$ are sufficient to exactly recover the Shapley values. In the SHAP library implementation of Kernel SHAP, this issue is addressed with a version of sampling without replacement: if there is a sufficient `sample budget' for a given set size, all sets of that size are sampled \citep{lundberg2017unified}. Doing so substantially improves Kernel SHAP as $m$ approaches $2^n$.
\end{itemize}

\subsection{Our Contributions}
At the time of our original submission, no non-asymptotic theoretical accuracy guarantees were known for Kernel SHAP when implemented with $m < 2^n-2$ row samples (which corresponds to $m$ evaluations of the value function, $v$). We address this issue by proposing \textbf{Leverage SHAP}, a lightweight modification of Kernel SHAP that 1) enjoys strong theoretical accuracy guarantees and 2) consistently outperforms Kernel SHAP in experiments, achieving a $50\%$ reduction in error on average (see Table \ref{tab:error2kshap}).

Leverage SHAP begins with the observation that a nearly optimal solution to the regression problem in \Eqref{eq:linear_regression_connection} can be obtained by sampling just $\tilde{O}(n)$ rows with probability proportional to their \emph{statistical leverage scores}, a natural measure for the ``importance'' or ``uniqueness'' of a matrix row \citep{Sarlos:2006,RauhutWard:2012,HamptonDoostan:2015,CohenMigliorati:2017}.\footnote{Technically, this fact is known for \emph{unconstrained} least square regression. Some additional work is needed to handle the linear constraint in \Eqref{eq:linear_regression_connection}, but doing so is relatively straightforward.} 
This fact immediately implies that, \emph{in principle}, we should be able to provably approximate $\boldsymbol{\phi}$ with a nearly-linear number of value function evaluations (one for each sampled row).

However, leverage scores are expensive to compute, naively requiring at least $O(2^n)$ time to write down for a matrix like $\mathbf{Z}$ with $O(2^n)$ rows. Our key observation is that this bottleneck can be avoided in the case of Shapley value estimation: we prove that the leverage scores of the projected matrix $\mathbf{A}=\mathbf{ZP}$ have a simple closed form that admits efficient sampling without ever writing them all down. Concretely, we show that the leverage score of the row corresponding to any subset $S\subset [n]$ is proportional to $\binom{n}{|S|}^{-1}$.

This suggests a similar, but meaningfully different sampling distribution than the one used by Kernel SHAP (see \Eqref{eq:kernel_shap_weights}).
Since all subsets of a given size have the same leverage score, we can efficiently sample proportional to the leverage scores by sampling a random size $s$ uniformly from $\{1,\ldots,n-1\}$ then selecting a subset $S$ of size $s$ uniformly at random.

\begin{theorem}
\label{thm:main}
    For any $\epsilon > 0$ and constant $\delta > 0$, the Leverage SHAP algorithm uses $m~=~O(n\log (\frac{n}{\delta}) + n\frac1{\epsilon\delta})$ evaluations of $v$ in expectation %
    and $O(mn^2)$ additional runtime to return estimated Shapley values $\tilde{\boldsymbol{\phi}}$ satisfying $\langle \tilde{\boldsymbol{\phi}}, \mathbf{1} \rangle = v([n]) - v(\emptyset)$ and, with probability $1-\delta$,
    \begin{align}
    \label{eq:main_objective_value}
        \| \mathbf{Z} \tilde{\boldsymbol{\phi}} - \mathbf{y} \|_2^2
        \leq (1 + \epsilon) 
        \| \mathbf{Z} \boldsymbol{\phi} - \mathbf{y} \|_2^2.
    \end{align}
\end{theorem}
In words, Theorem \ref{thm:main} establishes that, with a near-linear number of function evaluations, we can compute approximate Shapley values whose \emph{objective value} is close to the true Shapley values. We also require $O(mn^2)$ additional runtime to solve the linearly constrained regression problem in \Eqref{eq:subsampled_solution}, although this cost can be reduced in practice using, e.g., iterative methods.

By leveraging the fact that $\mathbf{Z}$ is a well-conditioned matrix, the bound in \Eqref{eq:main_objective_value} also implies a bound on the average squared error, $\| \tilde{\boldsymbol{\phi}} - \boldsymbol{\phi} \|_2^2$, which is provided in Section \ref{sec:theory_guarantees}. (In Appendix \ref{app:hardness}, we discuss why estimating Shapley values with multiplicative error is NP-hard.)

Beyond our theoretical results, we also show that leverage score sampling can be naturally combined with paired sampling, without sacrificing theoretical guarantees. Moreover, we show that a natural ``without replacement'' version of leverage score sampling leads to additional accuracy improvements when $m$ is large. Overall, these simple optimizations lead to an algorithm that consistently outperforms the optimized version of Kernel SHAP available in the SHAP Library (this `Optimized Kernel SHAP' algorithm also uses paired sampling and sampling without replacement). We illustrate this point in Figure \ref{fig:detailed}. 

More extensive experiments are included in Section \ref{sec:experiments}. We find that the improvement of Leverage SHAP over Kernel SHAP is especially substantial in settings where $n$ is large in comparison to $m$ and when we only have access to noisy estimates of the Shapley values. This is often the case in applications to explainable AI where, as discussion earlier, $v(S)$ is an expectation over functions involving random baselines and estimated via a finite sample.
\begin{figure*}[t!]
    \centering
    \includegraphics[width=\linewidth]{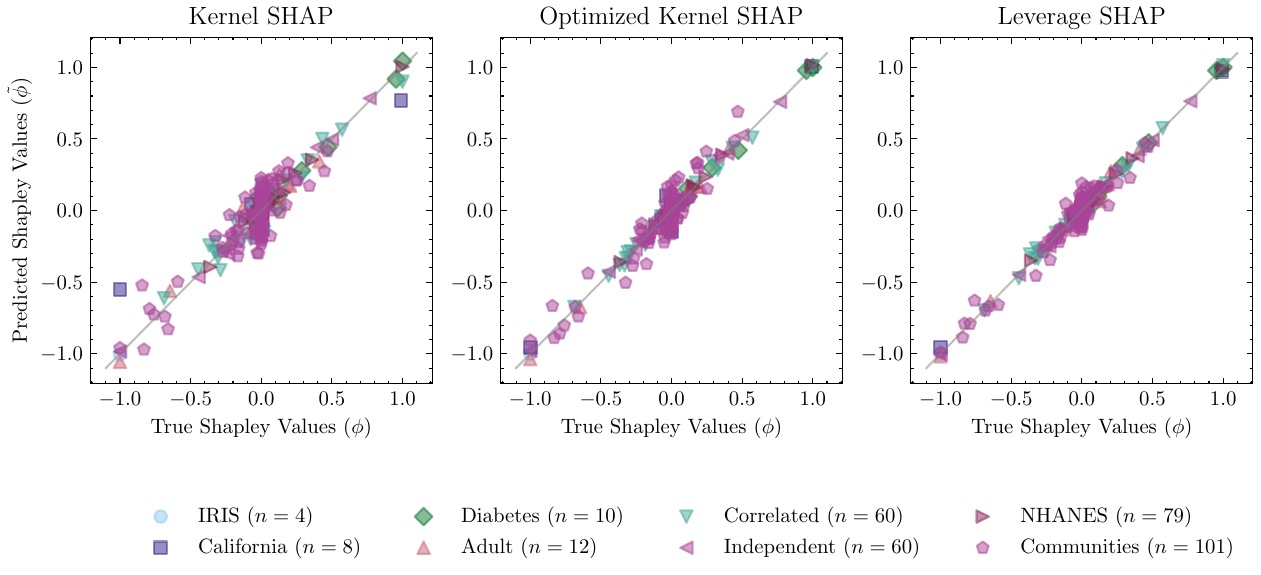}
    \caption{Predicted versus true Shapley values for all features in 8 datasets (we use $m=5n$ samples for this experiment).
    Points near the identity line indicate that the estimated Shapley value is close to its true value. The plots suggest that our Leverage SHAP method is more accurate than the baseline Kernel SHAP algorithm, as well as the optimized Kernel SHAP implementation available in the SHAP library. We corroborate these findings with more experiments in Section \ref{sec:experiments}.}
    \label{fig:detailed}
    \vspace{-1em}
\end{figure*}

\subsection{Related Work}

\paragraph{Shapley Values Estimation.} As discussed, naively computing Shapley values requires an exponential number of evaluations of $v$. While more efficient methods exist for certain structured value functions (see references in Section \ref{sec:efficient_shap_vals}), for generic functions, faster algorithms involve some sort of approximation. 
The most direct way of obtaining an approximation is to approximate the summation definition of Shapley values (Equation \ref{eq:shap_standard}), which involves $O(2^n)$ terms for each subset of $[n]$, with a random subsample of subsets \citep{castro2009polynomial,strumbelj2010efficient,vstrumbelj2014explaining}. The first methods for doing so use a different subsample for each player $i$. 

A natural alternative is to try selecting subsets in such a way that allows them to be reused across multiple players \citep{illes2019estimation,mitchell2022sampling}. However, since each term in the summation involves both $v(S)$ and $v(S\cup \{i\})$, it is difficult to achieve high levels of sample reuse when working with the summation definition. One option is to split the sum into two, and separately estimate $\sum_{S \subseteq [n] \setminus \{i\}} v(S \cup \{i\})/\binom{n-1}{|S|}$  and $\sum_{S \subseteq [n] \setminus \{i\}}v(S)/\binom{n-1}{|S|}$. This allows substantial subset reuse, but tends to perform poorly in practice due to high variance in the individual sums \citep{wang2023databanzhaf}.  
By solving a global regression problem that determines the entire $\boldsymbol{\phi}$ vector, Kernel SHAP can be viewed as a more effective way of reusing subsets to obtain Shapley values for all players.

In addition to the sampling methods referenced above, \cite{covert2020improving} propose a modified Kernel SHAP algorithm, prove it is unbiased, and compute its asymptotic variance.
However, they find that the modified version performs worse than Kernel SHAP.
Additionally, we note that some recent work in the explainable AI setting takes advantage of the fact that we often wish to evaluate feature impact for a large number of different input vectors $\mathbf{x}\in \R^n$, which each induce their own set function $v$. In this setting, it is possible to leverage information gathered for one input vector to more efficiently compute Shapley values for another, further reducing sample cost \citep{schwab2019cxplain,jethani2021fastshap}.
This setting is incomparable to ours; like Kernel SHAP, Leverage SHAP works in the setting where we have black-box access to a \emph{single} value function $v$.

Subsequent work by \citet{chen2025unified} develops a unified framework that includes Kernel SHAP and Leverage SHAP, and establishes non-asymptotic guarantees for Kernel SHAP as well.

\paragraph{Leverage Scores.} As discussed, leverage scores are a natural measure of importance for matrix rows. They are widely used as importance sampling probabilities in randomized matrix algorithms for problems ranging from regression \citep{CohenLeeMusco:2015,AlaouiMahoney:2015}, to low-rank approximation \citep{DrineasMahoneyMuthukrishnan:2008,cohen2017input,MuscoMusco:2017,rudi2018fast}, to graph sparsification \citep{SpielmanSrivastava:2011}, and beyond \citep{AgarwalKakadeKidambi:2020}.

More recently, leverage score sampling has been applied extensively in active learning as a method for selecting data examples for training \citep{CohenMigliorati:2017,AvronKapralovMusco:2019,ChenPrice:2019a,ErdelyiMuscoMusco:2020,ChenDerezinski:2021,GHM22,focs2022,CardenasAdcockDexter:2023,ShimizuChengMusco:2024}. The efficient estimation of Shapley values via the regression problem in \Eqref{eq:kernel_shap_weights} can be viewed as an active learning problem, since our primary cost is in observing entries in the target vector $\mathbf{y}$, each of which corresponds to an expensive value function evaluation.

For active learning of linear models with $n$ features under the $\ell_2$ norm, the $O(n\log n + n/\epsilon)$ sample complexity required by leverage score sampling is known to be near-optimal. More complicated sampling methods can achieve $O(n/\epsilon)$, which is optimal in the worst-case \citep{BatsonSpielmanSrivastava:2012, ChenPrice:2019a}. However, it is not clear how to apply such methods efficiently to an exponentially tall matrix, as we manage to do for leverage score sampling.

\section{Background and Notation}

\label{sec:background}
\paragraph{Notation.}
Lowercase letters represent scalars, bold lowercase letters vectors, and bold uppercase letters matrices.
We use the set notation $[n] = \{1, \dots, n\}$ and $\emptyset=\{\}$.
We let $\mathbf{0}$ denote the all zeros vector, $\mathbf{1}$ the all ones vector, and $\mathbf{I}$ the identity matrix, with dimensions clear from context.
For a vector $\mathbf{x}$, $x_i$ is the $i^\text{th}$ entry (non-bold to indicate a scalar).
For a matrix $\mathbf{X} \in \mathbb{R}^{\rho \times n}$, $[\mathbf{X}]_i \in \mathbb{R}^{1 \times n}$ is the $i^\text{th}$ row. 
For a vector $\mathbf{x}$, $\|\mathbf{x}\|_2 = (\sum_i x_i^2)^{1/2}$ denotes the Euclidean ($\ell_2$) norm and $\|\mathbf{x}\|_1 = \sum_i |x_i|$ is the $\ell_1$ norm. 
For a matrix $\mathbf{X}\in \R^{n\times m}$, $\mathbf{X}^+$ denotes the Moore–Penrose pseudoinverse.

\paragraph{Preliminaries.}
Any subset $S \subseteq [n]$ can be represented by a binary indicator vector $\mathbf{z} \in \{0,1\}^n$, and we use
$v(S)$ and $v(\mathbf{z})$ interchangeably. We construct the matrix $\mathbf{Z}$ and target vector $\mathbf{y}$ appearing in \Eqref{eq:linear_regression_connection} by indexing rows by all $\mathbf{z} \in \{0,1\}^n$ with $0 < \| \mathbf{z} \|_1 < n$:
\begin{itemize}
    \item Let $\mathbf{Z} \in \mathbb{R}^{(2^n-2) \times n}$ be a matrix with $[\mathbf{Z}]_{\mathbf{z}} = \sqrt{w(\| \mathbf{z}\|_1)} \mathbf{z}^\top$.
    \item Let $\mathbf{y} \in \mathbb{R}^{2^n-2}$ be the vector where $[\mathbf{y}]_{\mathbf{z}} = \sqrt{w(\| \mathbf{z}\|_1)} (v(\mathbf{z}) - v(\mathbf{0}))$.
\end{itemize}
Above, $w(s)= \left(\binom{n}{s} s (n-s)\right)^{-1}$ is the same weight function defined in  \Eqref{eq:kernel_shap_weights}.

As discussed, the Kernel SHAP method is based on an equivalence between Shapley values and the solution of a constrained regression problem involving $\mathbf{Z}$ and $\mathbf{y}$. Formally, we have:
\begin{restatable}[Equivalence \citep{lundberg2017unified,charnes1988extremal}]{lemma}{equivalence}\label{lemma:equivalence}
   \begin{align}
       \boldsymbol{\phi} &= \argmin_{\mathbf{x}: \langle \mathbf{x}, \mathbf{1} \rangle = v(\mathbf{1}) - v(\mathbf{0})} \| \mathbf{Z x - y} \|_2^2  \\
       &= \argmin_{\mathbf{x}: \langle \mathbf{x}, \mathbf{1} \rangle  = v(\mathbf{1}) - v(\mathbf{0})} 
       \sum_{\mathbf{z}: 0 < \| \mathbf{z} \|_1 < n} w(\|\mathbf{z}\|_1)\cdot \left(\langle \mathbf{z},\mathbf{x}\rangle - (v(\mathbf{z})-v(\mathbf{0}))\right)^2
       \label{eq:summation_form}.
   \end{align} 
\end{restatable}
For completeness, we provide a self-contained proof of Lemma \ref{lemma:equivalence} in Appendix \ref{appendix:equivalence}. 
The form in Equation \ref{eq:summation_form} inspires the heuristic choice to sample sets with probabilities proportional to $w(\|\mathbf{z}\|_1)$ in Kernel SHAP, as larger terms in the sum should intuitively be sampled with higher probability to reduce variance of the estimate. However, as discussed, a more principled way to approximately solve least squares regression problems via subsampling is to use probabilities proportional to the statistical leverage scores. Formally, these scores are defined as follows:
\begin{definition}[Leverage Scores]
    Consider a matrix $\mathbf{X} \in \mathbb{R}^{\rho \times n}$.
    For $i \in [\rho]$, the leverage score of the $i^{\text{th}}$ row $[\mathbf{X}]_i\in \R^{1\times n}$ is
    $\ell_i \vcentcolon =
    [\mathbf{X}]_i (\mathbf{X^\top X})^{+} [\mathbf{X}]_i^\top.$
\end{definition}

\section{Leverage SHAP}

Statistical leverage scores are traditionally used to approximately solve unconstrained regression problems. Since
Shapley values are the solution to a linearly constrained problem, we first reformulate this into an unconstrained problem. Ultimately, we sample by leverage scores of the reformulated problem, which we prove have a simple closed form that admits efficient sampling.

Concretely, we have the following equivalence:

\begin{restatable}[Constrained to Unconstrained]{lemma}{freedomlemma}\label{lemma:freedom}
Let $\mathbf{P}$ be the projection matrix $\mathbf{I} - \frac1{n} \mathbf{11}^\top$.
Define $\mathbf{A} = \mathbf{Z P}$ and $\mathbf{b} =\mathbf{y} - \mathbf{Z 1} \frac{v(\mathbf{1}) - v(\mathbf{0})}{n}.$
Then
\begin{align}
    \argmin_{\mathbf{x}: \langle \mathbf{x}, \mathbf{1} \rangle = v(\mathbf{1}) - v(\mathbf{0})} 
    \| \mathbf{Z x - y} \|_2^2
    = \argmin_{\mathbf{x}} \left\| \mathbf{Ax - b} \right\|_2^2
    + \mathbf{1} \frac{v(\mathbf{1}) - v(\mathbf{0})}{n}.
\end{align}
Further, we have that $\min_{\mathbf{x}: \langle \mathbf{x}, \mathbf{1} \rangle = v(\mathbf{1}) - v(\mathbf{0})} 
\| \mathbf{Z x - y} \|_2^2
= \min_{\mathbf{x}} \left\| \mathbf{Ax - b} \right\|_2^2.$    
\end{restatable}

When there are multiple $\mathbf{x}$ that minimize the objective $\left\| \mathbf{Ax - b} \right\|_2^2$, we define $\argmin_{\mathbf{x}} \left\| \mathbf{Ax - b} \right\|_2^2$ as the $\mathbf{x}$ that also minimizes $\| \mathbf{x}\|_2^2$.

Lemma \ref{lemma:freedom} is proven in Appendix \ref{appendix:constrained_solution}. Using the equivalent unconstrained regression problem, we consider methods that construct a sampling matrix $\mathbf{S}$ and return:
\begin{align}
    \tilde{\boldsymbol{\phi}} 
    = \argmin_{\mathbf{x}} \| \mathbf{S A x - S b} \|_2^2
    + \mathbf{1} \frac{v(\mathbf{1}) - v(\mathbf{0})}{n} .
\end{align}
The main question is how to build $\mathbf{S}$. Our Leverage SHAP method does so by sampling with probabilities proportional to the leverage scores of $\textbf{A}$. Since naively these scores would be intractable to compute, requiring $O(\rho n^2)$ time, where $\rho = 2^n - 2$ is the number of rows in $\mathbf{A}$, our method rests on the derivation of a simple closed form expression for the leverage scores, which we prove below. 

\subsection{Analytical Form of Leverage Scores}
\begin{lemma}\label{lemma:leverage_scores}
Let $\mathbf{A}$ be as defined in Lemma \ref{lemma:freedom}. The leverage score of the row in $\mathbf{A}$ with index $\mathbf{z}\in \{0,1\}^n$, where $0 <\|\mathbf{z}\|_1 < n$, is equal to
    $\ell_{\mathbf{z}} = \binom{n}{\|\mathbf{z}\|_1}^{-1}$.
\end{lemma}

The proof of Lemma \ref{lemma:leverage_scores} depends on an explicit expression for the matrix $\mathbf{A^\top A}$: 
\begin{lemma}\label{lemma:ATA_form}
Let $\mathbf{A}$ be as defined in Lemma \ref{lemma:freedom}. $\mathbf{A^\top A} = \frac1{n} \mathbf{P}$.
\end{lemma}
This fact will also be relevant later in our analysis, as it implies that $\mathbf{A^\top A}$ is well-conditioned. Perhaps surprisingly, all of its non-zero singular values are equal to $1/n$.
\begin{proof}[Proof of Lemma \ref{lemma:ATA_form}]
Recall that $\mathbf{A} = \mathbf{ZP}$, so $\mathbf{A}^\top \mathbf{A} = \mathbf{P^\top Z^\top}\mathbf{ZP}$. We start by deriving explicit expressions for the $(i,j)$ entry of 
$\mathbf{Z}^\top \mathbf{Z}$, denoted by $[\mathbf{Z}^\top \mathbf{Z}]_{(i,j)}$, for all $i\in [n]$ and $j\in [n]$. In particular, we can check that $[\mathbf{Z}^\top \mathbf{Z}]_{(i,j)}
= \sum_{\mathbf{z} \in \{0,1\}^n} z_i z_j w(\|\mathbf{z}\|_1).$
So, when $i\neq j$,
\begin{align}\label{eq:zz_entry_i_neq_j}
    [\mathbf{Z}^\top \mathbf{Z}]_{(i,j)} = \sum_{\mathbf{z}: z_i, z_j = 1} w(\|\mathbf{z}\|_1)
    &= \sum_{s=2}^{n-1} \binom{n-2}{s-2} w(s).
\end{align}
Let $c_n$ denote the above quantity, which does not depend on $i$ and $j$. I.e., $[\mathbf{Z}^\top \mathbf{Z}]_{(i,j)} = c_n$ for $i\neq j.$

For the case $i =j$, let $j'$ be any fixed index not equal to $i$. We have that
\begin{align}\label{eq:zz_entry_i_eq_j}
    [\mathbf{Z}^\top \mathbf{Z}]_{(i,i)} &= \sum_{\mathbf{z}: z_i = 1} w(\| \mathbf{z} \|_1)
    = \sum_{\mathbf{z}: z_i = 1, z_{j'}=0} w(\| \mathbf{z} \|_1)
    + \sum_{\mathbf{z}: z_i = 1, z_{j'}=1} w(\| \mathbf{z} \|_1)
    \nonumber\\&= \sum_{s=1}^{n-1} \binom{n-2}{s-1} w(s)
    + \sum_{s=2}^{n-1} \binom{n-2}{s-2} w(s)
    = \frac1{n}
    + c_n.
\end{align}
The last equality can be verified via a direct calculation. 
By Equations \ref{eq:zz_entry_i_neq_j} and \ref{eq:zz_entry_i_eq_j}, we have that $\mathbf{Z^\top Z} = \frac1{n} \mathbf{I} + c_n \mathbf{1} \mathbf{1}^\top$.
So we conclude that: 
\begin{align*}
    \mathbf{A}^\top \mathbf{A}=
    \mathbf{P^\top Z^\top Z P}
    &= \left(\mathbf{I} - \frac1{n} \mathbf{11^\top}\right)
    \left(\frac1{n} \mathbf{I} + c_n \mathbf{11}^\top\right)
    \left(\mathbf{I} - \frac1{n} \mathbf{11^\top}\right)
    = \frac1{n} \left(\mathbf{I} - \frac1{n} \mathbf{11^\top}\right). & &\qedhere
\end{align*}
\end{proof}

With Lemma \ref{lemma:ATA_form}, we are now ready to analytically compute the leverage scores.

\begin{proof}[Proof of Lemma \ref{lemma:leverage_scores}]
The leverage scores of $\mathbf{A}$ are given by $
\ell_{\mathbf{z}} =
[\mathbf{ZP}]_{\mathbf{z}}
\left( \mathbf{A^\top A}\right)^{+} 
[\mathbf{ZP}]_{\mathbf{z}}^\top$.
By Lemma \ref{lemma:ATA_form}, we have that $\mathbf{A^\top A} = \frac1{n} \left(\mathbf{I} - \frac1{n} \mathbf{11}^\top\right)$ and thus $(\mathbf{A^\top A})^+ = n\mathbf{I} - \mathbf{11}^\top$.
Recall that $[\mathbf{Z}]_\mathbf{z} = \sqrt{w(\| \mathbf{z} \|_1)} \mathbf{z}^\top$ so $[\mathbf{Z P}]_{\mathbf{z}}  =  \sqrt{w(\| \mathbf{z} \|_1)} \mathbf{z}^\top \left(\mathbf{I} - \frac1{n} \mathbf{1 1^\top}\right)$. We thus have:
\begin{align}
    \ell_{\mathbf{z}}
    &= \sqrt{w(\|\mathbf{z}\|_1)} \left(\mathbf{z} - \frac{\|\mathbf{z}\|_1}{n} \mathbf{1}\right)^\top
    (n \mathbf{I} - \mathbf{11}^\top)
    \sqrt{w(\|\mathbf{z}\|_1)} \left(\mathbf{z} - \frac{\|\mathbf{z}\|_1}{n} \mathbf{1}\right)
    \nonumber \\ &
    =w(\|\mathbf{z}\|_1) (n\| \mathbf{z} \|_1 - \| \mathbf{z} \|_1^2)
    = w(\|\mathbf{z}\|_1)
    \| \mathbf{z} \|_1 (n - \| \mathbf{z} \|_1)
    = \binom{n}{\| \mathbf{z} \|_1}^{-1}.\qedhere
\end{align}
\end{proof}

Notice that the leverage scores simply correspond to sampling every row inversely proportional to the number of sets of that size.
In retrospect, this is quite intuitive given the original definition of Shapley values in \Eqref{eq:shap_standard}.
As shown in Figure \ref{fig:sampling_probs}, unlike the Kernel SHAP sampling distribution, leverage score sampling ensures that all subset sizes are equally represented in our sample from $\mathbf{A}$. 

\begin{figure*}[tb!]
    \centering
    \includegraphics[width=\linewidth]{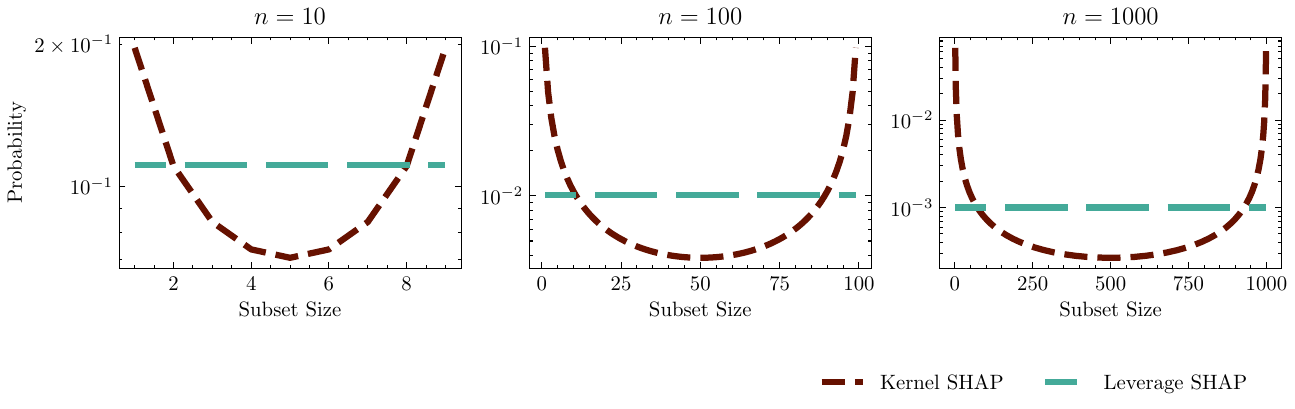}
    \caption{
    Let $S$ be a subset sampled with the Kernel SHAP or Leverage SHAP probabilities.
    The plots show the distribution of the set size $|S|$ for different $n$.
    As in in the definition of Shapley values, the leverage score distribution places equal \emph{total weight} on each subset size, contrasting with Kernel SHAP which over-weights small and large subsets.
    }
    \label{fig:sampling_probs}
    \vspace{-1em}
\end{figure*}

\subsection{Our Algorithm}\label{sec:our_algorithm}
With Lemma \ref{lemma:leverage_scores} in place, we are ready to present Leverage SHAP, which is given as Algorithm \ref{alg:leverage_shap}. In addition to leverage score sampling, the algorithm incorporates paired sampling and sampling without replacement, both of which are used in optimized implementations of Kernel SHAP.

For paired sampling, the idea is to sample rows with probabilities proportional to the leverage scores, but in a \emph{correlated way}: any time we sample index $\mathbf{z}$, we also select its complement, $\bar{\mathbf{z}}$, where $\bar{z}_i = 1 - z_i$ for $i \in [n]$. Note that, by the symmetry of $\mathbf{A}$'s leverage scores, $\ell_{\mathbf{z}} = \ell_{\bar{\mathbf{z}}}$.
Similarly, $w(\| \mathbf{z} \|_1) = w(\| \bar{\mathbf{z}} \|_1)$. Moving forward, let $\mathcal{Z}$ denote the set of pairs $(\mathbf{z},\bar{\mathbf{z}})$ where $0 < \|\mathbf{z}\|_1 < n$.

To perform paired sampling without replacement, we select indices $(\mathbf{z}, \bar{\mathbf{z}})$ independently at random with probability $\min(1, c (\ell_{\mathbf{z}}+\ell_{\bar{\mathbf{z}}})) = \min(1, 2c \ell_{\mathbf{z}})$, where $c >0$ is an oversampling parameter. All rows that are sampled are included in a subsampled matrix $\mathbf{Z'P}$ and reweighted by the inverse of the probability with which they were sampled. Each selected pair contributes two rows, so the expected number of row samples in $\mathbf{Z'P}$ is equal to $2\sum_{(\mathbf{z}, \bar{\mathbf{z}})} \min(1, 2c \ell_{\mathbf{z}})$. We choose $c$ via binary search so that this expectation equals  $m -2$, where $m$ is our target number of value function evaluations (two evaluations are reserved to compute $v(\mathbf{1}) - v(\mathbf{0}))$. I.e., we choose $c$ to solve the equation:
\begin{align}\label{eq:expected_samples}
    m - 2 
    = \sum_{s=1}^{n-1} \min\left(1, 2c \binom{n}{s}^{-1}\right) \binom{n}{s}.
\end{align}

Note that our procedure is different from the with-replacement Kernel SHAP procedure described in Section \ref{sec:efficient_shap_vals}. Sampling without replacement requires more care to avoid iterating over the exponential number of pairs in $\mathcal{Z}$. Fortunately, we can exploit the fact that there are only $n-1$ different set sizes, and all sets of the same size have the same leverage score. This allows us to determine how many sets of a different size should be sampled (by drawing a binomial random variable), and we can then sample the required number of sets of each size uniformly at random. Ultimately, we can collect $m$ samples in $O(mn^2)$ time. Details are deferred to Appendix \ref{appendix:bernoulli_code}. 

Because we sample independently without replacement, it is possible that the number of function evaluations $2| \mathcal{Z}' |+2$ used by Algorithm \ref{alg:leverage_shap} exceeds $m$.
However, because this number is itself a sum of independent random variables, it tightly concentrates around its expectation, $m$.
For an even feasible budget $m\leq2^n$, our released implementation makes the number of evaluations deterministically $m$ (odd budgets are rounded up): instead of drawing $m_s$ from a Binomial on Line \ref{line:binomial_sample} of Algorithm \ref{alg:bernoulli_sampling}, it sets $m_s$ to the expectation of that Binomial, rounded to integers so that the total number of sampled sets is exactly $m-2$.
This stratifies the sample by set size, in the spirit of stratified SVARM \citep{kolpaczki2024approximating}.
Since this version no longer samples sets independently, our theoretical analysis does not apply to it verbatim; we believe the guarantees still hold and leave a formal analysis to future work.
The experiments in Section \ref{sec:experiments} use the deterministic version; Appendix \ref{appendix:ablation} also reports a corrected Binomial-count implementation (``Leverage SHAP (Binomial)'') for comparison; its remaining implementation differences from Algorithm \ref{alg:bernoulli_sampling} are described there.

\begin{algorithm}[tb!]
\caption{Leverage SHAP}\label{alg:leverage_shap}
\begin{algorithmic}[1]
\Input $n$: number of players, $v$: set function, $m$: target number of function evaluations (at least $n$)
\Output $\tilde{\boldsymbol{\phi}}$: approximate Shapley values
\State $m \gets \min(m, 2^n)$ \Comment{$2^n$ samples to solve exactly}
\State Find $c$ via binary search so that $m-2 = \sum_{s=1}^{n-1} \min(\binom{n}{s}, 2c) $ \Comment{Equation \ref{eq:expected_samples}}
\State $\mathcal{Z}' \gets$ \texttt{BernoulliSample}$(n, c)$ \Comment{Sample from $\mathcal{Z}$ without replacement (Algorithm \ref{alg:bernoulli_sampling})}
\State $i \gets 0$, $\mathbf{W} \gets \mathbf{0}_{2|\mathcal{Z}'| \times 2|\mathcal{Z}'|}$, $\mathbf{Z}' \gets \mathbf{0}_{2|\mathcal{Z}'| \times n}$
\For{$(\mathbf{z}, \bar{\mathbf{z}}) \in \mathcal{Z}'$}
\Comment{Build sampled regression problem}
\State $\mathbf{Z}'_{(i,)} \gets \mathbf{z}^\top$, $\mathbf{Z}'_{(i+1,)} \gets \bar{\mathbf{z}}^\top$
\State $W_{(i,i)} \gets W_{(i+1,i+1)} \gets \frac{w(\| \mathbf{z} \|_1)}{\min(1, 2c \cdot \ell_{\mathbf{z}})}$
\Comment{Correct weights in expectation}
\State $i \gets i+2$
\EndFor
\State $\mathbf{y'} \gets v(\mathbf{Z}') - v(\mathbf{0}) \mathbf{1}$ \Comment{$v(\mathbf{Z}')$ evaluates $v$ on $2|\mathcal{Z}'|$ inputs, $m-2$ in expectation}
\State $\mathbf{b'} \gets \mathbf{y'} - \frac{v(\mathbf{1}) - v(\mathbf{0})}{n} \mathbf{Z' 1}$
\State Compute $$\tilde{\boldsymbol{\phi}}^\perp \gets \argmin_{\mathbf{x}} \| \mathbf{W}^\frac12 \mathbf{Z'} \mathbf{P x} - \mathbf{W}^\frac12 \mathbf{b'} \|_2^2$$
\Comment{Standard least squares}
\State \textbf{return} $\tilde{\boldsymbol{\phi}} \gets \tilde{\boldsymbol{\phi}}^\perp + \frac{v(\mathbf{1}) - v(\mathbf{0})}{n} \mathbf{1}$
\end{algorithmic}
\end{algorithm}

\section{Theoretical Guarantees}\label{sec:theory_guarantees}
    
As discussed, Leverage SHAP offers strong theoretical approximation guarantees. 
Our main result is Theorem \ref{thm:main}. Its core is Theorem \ref{thm:leverage_shap_paired} in Appendix \ref{appendix:theory}, where we show that, with $m = O(n \log (n/\delta) + \frac{n}{\epsilon\delta})$ samples in expectation, Algorithm \ref{alg:leverage_shap} returns a solution $\tilde{\boldsymbol{\phi}}$ that, with probability $1-\delta$, satisfies
\begin{align}
    \| \mathbf{A} \tilde{\boldsymbol{\phi}} - \mathbf{b} \|_2^2
    \leq (1 + \epsilon) 
    \| \mathbf{A} \boldsymbol{\phi} - \mathbf{b} \|_2^2.
\end{align}
We prove this guarantee in Appendix \ref{appendix:theory}.
The proof modifies the standard analysis of leverage scores for active least squares regression, which requires two components: 1) a \emph{subspace embedding guarantee}, proven with a matrix Chernoff bound applied to a sum of rank-$1$ random matrices, and 2) an \emph{approximate matrix-multiplication guarantee}, proven with a second moment analysis.
We replace these steps with 1) a matrix Bernstein bound applied to a sum of rank-$2$ random matrices and 2) a block-approximate matrix multiplication guarantee.

While well-motivated by the connection between Shapley values and linear regression, the objective function in Theorem \ref{thm:main} is not intuitive.
Instead, we may be interested in the $\ell_2$-norm error between the true Shapley values and the estimated Shapley values.
Fortunately, we can use special properties of our problem to show that Theorem \ref{thm:main} implies the following corollary on the $\ell_2$-norm error.

\begin{restatable}{corollary}{corollarynorm}\label{coro:l2norm}
Suppose $\tilde{\boldsymbol{\phi}}$ satisfies the efficiency constraint $\langle\tilde{\boldsymbol{\phi}},\mathbf{1}\rangle=v(\mathbf{1})-v(\mathbf{0})$ and
$\| \mathbf{A} \tilde{\boldsymbol{\phi}} - \mathbf{b} \|_2^2 \leq (1+\epsilon) \| \mathbf{A \boldsymbol{\phi} - b} \|_2^2$.
Assume $\|\mathbf{A}\boldsymbol{\phi}\|_2>0$ and let $\gamma = \| \mathbf{A} \boldsymbol{\phi} - \mathbf{b} \|_2^2 / \| \mathbf{A} \boldsymbol{\phi} \|_2^2$.
Then
\begin{align*}
\| \boldsymbol{\phi} - \tilde{\boldsymbol{\phi}} \|_2^2 
\leq \epsilon \gamma \| \boldsymbol{\phi}\|_2^2.
\end{align*}
\end{restatable}

The proof of Corollary \ref{coro:l2norm} appears in Appendix \ref{appendix:corollary_proof}.
The statement of the corollary is in terms of a problem-specific parameter $\gamma$ which intuitively measures how well the \textit{optimal} coefficients $\boldsymbol{\phi}$ can reach the target vector $\mathbf{b}$.
In Appendix \ref{appendix:gamma}, we explore how the performance varies with $\gamma$ in practice.
We find (see e.g., Figures \ref{fig:gamma-l2}, \ref{fig:ablation_gamma-l2}, \ref{fig:ablation_gamma-objective}) that performance does erode as $\gamma$ increases for all regression-based algorithms, suggesting that $\gamma$ is not an artifact of our analysis.

\section{Experiments}
\label{sec:experiments}

In the experiments, we evaluate Leverage SHAP and Kernel SHAP based on how closely they align with the ground truth Shapley values.
We primarily use the relative squared $\ell_2$ error i.e., $\| \bm{\phi} - \tilde{\bm{\phi}}\|_2^2/\| \bm{\phi}\|_2^2$ as the error metric.
We run our experiments on eight popular datasets from the SHAP library \citep{lundberg2017unified}.
We find that Leverage SHAP outperforms the highly optimized Kernel SHAP, achieving a $50\%$ reduction in error on average (see Table \ref{tab:error2kshap}).
In addition, we find that Leverage SHAP consistently outperforms Permutation SHAP (see the additional experiments in Section \ref{appendix:ablation}).

\begin{figure}[tb!]
    \centering
    \includegraphics[width=\linewidth]{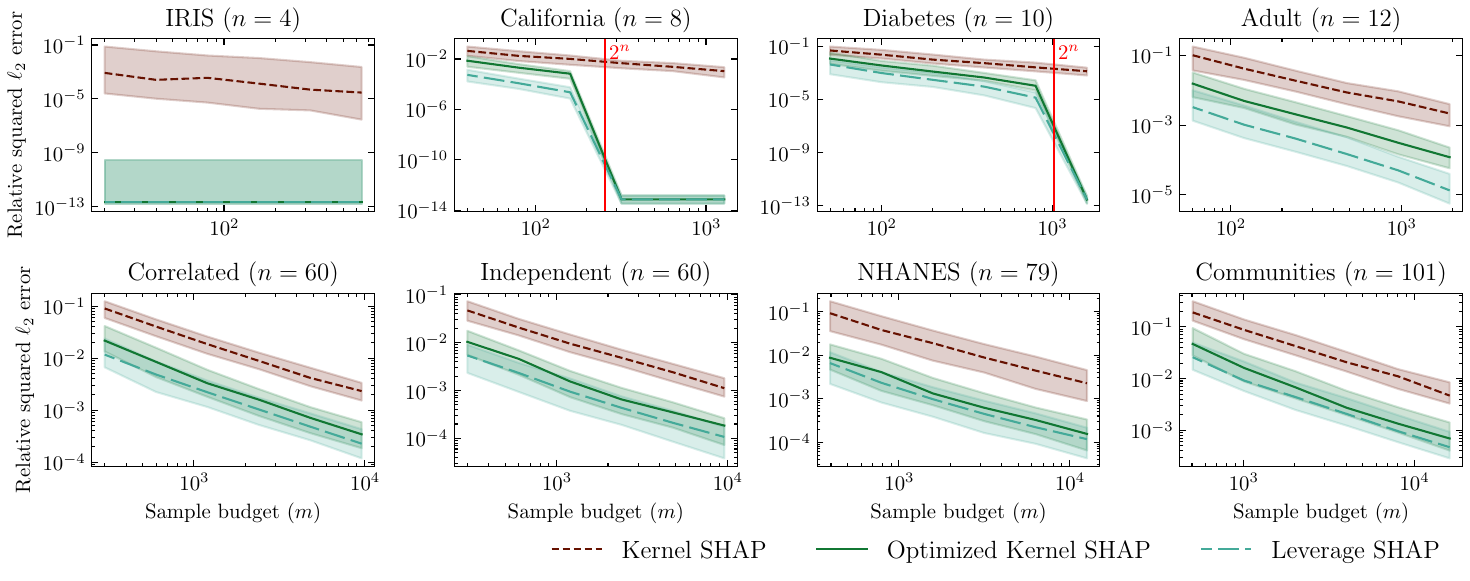}
    \caption{The relative squared $\ell_2$ error between the estimated Shapley values and ground truth Shapley values as a function of sample size. The lines report the median error while the shaded regions encompass the first to third quartile over 100 runs. Besides the setting where the exact Shapley values can be recovered, Leverage SHAP reliably outperforms even Optimized Kernel SHAP: on Correlated and Independent, the third quartile of the Leverage SHAP error is roughly the median error of Optimized Kernel SHAP; on NHANES and Communities it is at most about $1.5$ times the median; and on the small-$n$ datasets it is below the median.}
    \label{fig:size-l2}
    \vspace{-1em}
\end{figure}

\textbf{Implementation Details.}
In order to compute the ground truth Shapley values for large values of $n$, we use a tree-based model for which we can compute the exact Shapley values efficiently using Tree SHAP \citep{lundberg2017unified}. All of our code is written in Python and can be found on Github\footnote{\url{github.com/rtealwitter/leverageshap}}. We use the SHAP library for the Optimized Kernel SHAP implementation, Tree SHAP, and the datasets. We use XGBoost for training and evaluating trees, with 100 estimators and maximum depth 4; other parameters use their defaults.

\textbf{Additional experiments.} So that we do not overcrowd the plots, we show performance on Kernel SHAP, the Optimized Kernel SHAP implementation, and Leverage SHAP in the main body.
We also run all of the experiments with the ablated estimators in Appendix \ref{appendix:ablation}.
The ablation experiments (see e.g., Figure \ref{fig:ablation_size-l2}) suggest that Bernoulli sampling (sampling without replacement) has little effect for large $n$ but reduces the error of Leverage SHAP for small $n$ by a factor between $1.1$ and $10$ (about $2$ in the typical setting) that grows as $m$ approaches $2^n$, while paired sampling reduces the error by a factor of roughly four or more in every setting of the sample size experiment.
Because it is more interpretable, we report the relative squared $\ell_2$ error metric to compare the estimated to the ground truth Shapley values. We also report the linear objective error for similar experiments in Appendix \ref{appendix:objective}.

Figure \ref{fig:size-l2} plots the performance for Kernel SHAP, the Optimized Kernel SHAP, and Leverage SHAP as the number of samples varies (we set $m=5n, 10n, 20n, \ldots, 160n$).
For each estimator, the line is the median error and the shaded region encompasses the first and third quartile over 100 random runs.\footnote{We report the median and quartiles instead of the mean and standard deviation since the mean minus standard deviation is negative for the small error values in our experiments.}
Kernel SHAP gives the highest error in all settings of $m$ and $n$, pointing to the importance of the paired sampling and sampling without replacement optimizations in both Optimized Kernel SHAP and Leverage SHAP.
Because both methods sample without replacement, Optimized Kernel SHAP and Leverage SHAP achieve essentially machine precision as $m$ approaches $2^n$.
For all other settings,
Leverage SHAP reliably outperforms Optimized Kernel SHAP; on Correlated and Independent the third quartile of the Leverage SHAP error is roughly at the median of Optimized Kernel SHAP, on NHANES and Communities it is at most about $1.5$ times the median, and on the small-$n$ datasets it is below the median.
The analogous experiment for the linear objective error depicted in Figure \ref{fig:ablation_size-objective} indicates the same findings.

\begin{table}[tb!]
    \centering
    \input{tables/main_shap_error}
    \caption{Summary statistics of the relative squared $\ell_2$ error for every dataset. We adopt the Olympic medal convention: \colorbox{gold!60}{gold}, \colorbox{silver!60}{silver} and \colorbox{bronze!60}{bronze} cells signify first, second and third best performance, respectively. Except for ties, Leverage SHAP gives the best performance across all settings.}
    \label{tab:l2-error}
    \vspace{-1em}
\end{table}

Table \ref{tab:l2-error} depicts the estimator performance for $m=10n$.
The table shows the mean, first quartile, second quartile, and third quartile of the error for each estimator on every dataset.
Except for ties in the setting where $m \geq 2^n$, Leverage SHAP gives the best performance followed by Optimized Kernel SHAP and then Kernel SHAP.

Beyond sample size, we evaluate how the estimators perform with noisy access to the set functions in Appendix \ref{app:noise}.
As in the sample size experiments, Leverage SHAP meets or exceeds the other estimators at low and moderate noise, suggesting the utility of the algorithm as a robust estimator in explainable AI.

\section{Conclusion}
We introduce Leverage SHAP, a principled alternative to Kernel SHAP, designed to provide provable accuracy guarantees with nearly linear model evaluations. The cornerstone of our approach is leverage score sampling, a powerful subsampling technique used in regression. To adapt this method for estimating Shapley values, we reformulate the standard Shapley value constrained regression problem and analytically compute the leverage scores of this reformulation. Leverage SHAP efficiently uses these scores to produce provably accurate estimates of Shapley values. Our method enjoys strong theoretical guarantees, which we prove by modifying the standard leverage score analysis to incorporate the empirically-motivated paired sampling and sampling without replacement optimizations. Through extensive experiments on eight datasets, we demonstrate that our algorithm outperforms even the optimized version of Kernel SHAP, establishing Leverage SHAP as a valuable tool in the explainable AI toolkit.

\section{Acknowledgements}
This work was supported by the National Science Foundation under Grant No. 2045590 and Graduate Research Fellowship Grant No. DGE-2234660.
We are grateful to Yurong Liu for introducing us to the problem of estimating Shapley values.

\newpage
\bibliography{references}
\bibliographystyle{iclr2025_conference}

\newpage
\appendix
\appendixpage
\tableofcontents
\newpage

\clearpage

\section{Proof of Approximation Guarantee}\label{appendix:theory}

In this section, we prove the following theorem.

\begin{theorem}\label{thm:leverage_shap_paired}
    Let $m = O(n \log (\frac{n}{\delta}) + n \frac1{\delta\epsilon})$.    
    Algorithm \ref{alg:leverage_shap} produces an estimate $\tilde{\boldsymbol{\phi}}$ such that, with probability $1-\delta$,
    \begin{align*}
    \| \mathbf{A} \tilde{\boldsymbol{\phi}} - \mathbf{b} \|_2^2 \leq (1+\epsilon) \| \mathbf{A \boldsymbol{\phi} - b} \|_2^2.
    \end{align*}
\end{theorem}

Because of the connection between the constrained and unconstrained regression problems described in Lemma \ref{lemma:freedom}, the theorem implies the theoretical guarantees in Theorem \ref{thm:main}.
The time complexity bound in Theorem \ref{thm:main} follows from Lemma \ref{lemma:time} in Appendix \ref{appendix:bernoulli_code}.

Consider the following leverage score sampling scheme where rows are selected in blocks.
(We use the word ``block'' instead of ``pair'' in the formal analysis for generality.)
Let $\Theta$ be a partition into blocks of equal size (size 2 in our case) where the leverage scores within a block are equal.
Block $\Theta_i$ is independently sampled with probability $p^+_i \vcentcolon= \min(1, c \sum_{k \in \Theta_i} \ell_k)$ for a constant $c$.
The constant $c$ is chosen so that the expected number of blocks sampled is $m$, i.e., $\sum_i p^+_i = m$.
Let $m'$ be the number of blocks sampled.
Let $\mathbf{S} \in \mathbb{R}^{|\Theta_i| m' \times \rho}$ be the (random) sampling matrix.
Each row of $\mathbf{S}$ corresponds to a row $k$ from a sampled block $i$:
every entry in the row is 0 except for the $k$th entry which is $1/\sqrt{p^+_i}$.

In order to analyze the solution returned by Algorithm \ref{alg:leverage_shap}, we will prove that the sampling matrix $\mathbf{S}$ preserves the spectral norm and Frobenius norm.

\begin{lemma}[Bernoulli Spectral Approximation]\label{lemma:spectral}
    Let $\mathbf{U} \in \mathbb{R}^{\rho \times n}$ be a matrix with orthonormal columns.
    Consider the block random sampling matrix $\mathbf{S}$ described above with rows sampled according to the leverage scores of $\mathbf{U}$.
    For $0<\epsilon\leq1$ and $0<\delta<1$, when $m = \Omega(n \log (n/\delta)/\epsilon^2)$,
    \begin{align}
        \| \mathbf{I} - \mathbf{U^\top S^\top S U} \|_2 \leq \epsilon
    \end{align}
    with probability $1-\delta$.
\end{lemma}

\begin{proof}[Proof of Lemma \ref{lemma:spectral}]
We will use the following matrix Bernstein bound (see e.g., Theorem 6.6.1 in \cite{tropp2015introduction}).

\begin{fact}[Matrix Bernstein]\label{fact:matrix_bernstein}
Consider a finite sequence $\{ \mathbf{X}_i \}$ of independent, random, Hermitian matrices with dimension $n$. Assume that $\mathbb{E}[\mathbf{X}_k] = \mathbf{0}$ and $\| \mathbf{X}_i \|_2 \leq L$ for all $i$.
Define $\mathbf{X} = \sum_{i} \mathbf{X}_i$ and let $V = \| \mathbb{E}[ \mathbf{X}^2] \|_2 = \| \sum_i \mathbb{E}[\mathbf{X}_i^2] \|_2$.
Then for any $\epsilon > 0$,
\begin{align}
    \Pr\left(\| \mathbf{X} \|_2 \geq \epsilon \right) \leq 2n \exp\left(\frac{-\epsilon^2/2}{V + L \epsilon/3}\right).
\end{align}
\end{fact}

Let $\mathbf{U}_{(k,)} \in \mathbb{R}^{1 \times n}$ be the $k$th row of $\mathbf{U}$.
Similarly, let $\mathbf{U}_{\Theta_i} \in \mathbb{R}^{|\Theta_i|\times n}$ be the matrix with rows $\mathbf{U}_{(k,)}$ for $k \in \Theta_i$.

We will choose $\mathbf{X}_i = \mathbf{U}_{\Theta_i}^\top \mathbf{U}_{\Theta_i} - \frac1{p^+_i} \mathbf{U}_{\Theta_i}^\top \mathbf{U}_{\Theta_i}\mathbbm{1}[i \textnormal{ selected}]$.
Then $\mathbb{E}[\mathbf{X}_i] = \mathbf{0}$ and $\mathbf{X} = \mathbf{I} - \mathbf{U}^\top \mathbf{S}^\top \mathbf{SU}$.

First, we will upper bound $\max_i \| \mathbf{X}_i \|_2$.
If $p^+_i=1$, then $\mathbf{X}_i=0$. If $p^+_i<1$ and block $i$ is not selected, then $\|\mathbf{X}_i\|_2\leq\sum_{k\in\Theta_i}\ell_k=p_i^+/c\leq1/c$. It remains to consider a selected block with $p^+_i = c \sum_{k \in \Theta_i} \ell_k < 1$.
Then
\begin{align}
    \| \mathbf{X}_i \|_2 \leq \left| 1 - \frac1{p^+_i} \right| \| \mathbf{U}_{\Theta_i}^\top \mathbf{U}_{\Theta_i} \|_2
    \leq \frac1{p^+_i} \sum_{k \in \Theta_i} \| \mathbf{U}_{(k,)} \|_2^2
    = \frac1{c} = L
\end{align}

Next, we will upper bound $\| \sum_i \mathbb{E}[\mathbf{X}_i^2] \|_2$. 
Again, notice that $\mathbb{E}[\mathbf{X}_i^2] = \mathbf{0}$ if $p^+_i = 1$ so we only need to consider the case where $p^+_i < 1$.
Then
\begin{align}\label{eq:expected_X}
    \sum_i \mathbb{E}[\mathbf{X}_i^2]
    = \sum_{i:p^+_i < 1} p^+_i \left(1 - \frac1{p^+_i}\right)^2 \left(\mathbf{U}_{\Theta_i}^\top \mathbf{U}_{\Theta_i} \right)^2 + (1-p^+_i) \left(\mathbf{U}_{\Theta_i}^\top \mathbf{U}_{\Theta_i} \right)^2.
\end{align}
Notice that $p^+_i (1-1/p^+_i)^2 + (1-p^+_i) = p^+_i - 2 + 1/p^+_i + 1 - p^+_i = 1/p^+_i - 1 \leq 1/p^+_i.$
Therefore
\begin{align}\label{eq:expected_X_bound}
    (\ref{eq:expected_X}) \preceq \sum_{i:p^+_i < 1} \frac1{p^+_i} \mathbf{U}_{\Theta_i}^\top \mathbf{U}_{\Theta_i} \mathbf{U}_{\Theta_i}^\top \mathbf{U}_{\Theta_i} 
\end{align}

Observe that entry $(k, k')$ of $\mathbf{U}_{\Theta_i}\mathbf{U}_{\Theta_i}^\top \in \mathbb{R}^{|\Theta_i| \times |\Theta_i|}$ is $\mathbf{U}_k \mathbf{U}_{k'}^\top$.
So the absolute value of each entry is $|\mathbf{U}_k \mathbf{U}_{k'}^\top| \leq \| \mathbf{U}_k \|_2 \| \mathbf{U}_{k'} \|_2 = \ell_k^{1/2} \ell_{k'}^{1/2}$ by Cauchy-Schwarz.
Define $\ell^{\max}_i = \max_{k \in \Theta_i} \ell_k$ and $\ell^{\min}_i = \min_{k \in \Theta_i} \ell_k$.
By the Gershgorin circle theorem, $\mathbf{U}_{\Theta_i}\mathbf{U}_{\Theta_i}^\top \preceq 2\ell^{\max}_i |\Theta_i|^2 \mathbf{I}$.
Equivalently, $\mathbf{A} \preceq \mathbf{B}$, $\mathbf{x}^\top \mathbf{A} \mathbf{x} \leq \mathbf{x}^\top \mathbf{B} \mathbf{x}$ for all $\mathbf{x}$.
Consider an arbitrary $\mathbf{z}$.
We have $\mathbf{z^\top C^\top A C z} \leq \mathbf{z^\top C^\top B C z}$ since $\mathbf{Cz}$ is some $\mathbf{x}$.
It follows that $\mathbf{U}_{\Theta_i}^\top\mathbf{U}_{\Theta_i}\mathbf{U}_{\Theta_i}^\top\mathbf{U}_{\Theta_i} \preceq 2\ell^{\max}_i |\Theta_i|^2 \mathbf{U}_{\Theta_i}^\top\mathbf{U}_{\Theta_i}$.
Then
\begin{align}
    (\ref{eq:expected_X_bound})
    \preceq \frac1{c} \sum_{i}
    \frac{2\ell^{\max}_i |\Theta_i|^2 \mathbf{U}_{\Theta_i}^\top\mathbf{U}_{\Theta_i}}
    {\sum_{k\in \Theta_i} \ell_k}
    \preceq \frac1{c} \max_{i} 2 |\Theta_i| \frac{\ell^{\max}_i}{\ell^{\min}_i} \mathbf{I}.
\end{align}
Since $|\Theta_i| \leq 2$ and the leverage scores in a block are all equal, $\|\sum_i \E[\mathbf{X}_i^2]\|_2 \leq 4/c$.

Recall that $c$ is chosen so that $m = \sum_{i} \min(1, c \sum_{k \in \Theta_i} \ell_k) \leq c \sum_i \sum_{k \in \Theta_i} \ell_k = c n$, where the last equality follows because the leverage scores sum to $n$.
So $1/c \leq n/m$.

Plugging in to Fact \ref{fact:matrix_bernstein}, we have that
\begin{align}
    \Pr\left(\| \mathbf{I} - \mathbf{U^\top S^\top S U} \|_2 \geq \epsilon \right) \leq 2n \exp\left(\frac{-m \epsilon^2/2n}{4 + \epsilon/3}\right) \leq \delta
\end{align}
provided $m = \Omega(n \log (n/\delta)/\epsilon^2)$.

\end{proof}

Next, we will prove that the sampling matrix preserves the Frobenius norm.

\begin{lemma}[Frobenius Approximation]\label{lemma:frobenius}
    Let $\mathbf{U} \in \mathbb{R}^{\rho \times n}$ be a matrix with orthonormal columns.
    Consider the block random sampling matrix $\mathbf{S}$ described above with rows sampled according to the leverage scores of $\mathbf{U}$.
    Let $\mathbf{V} \in \mathbb{R}^{\rho \times n'}$.
    As long as $m \geq \frac1{\delta \epsilon^2}$, then
    \begin{align}
        \left\| \mathbf{U^\top S^\top S V} - \mathbf{U^\top V} \right\|_F \leq \epsilon \| \mathbf{U}\|_F \|\mathbf{V} \|_F
    \end{align}
    with probability $1-\delta$.
\end{lemma}

\begin{proof}
We adapt Proposition 2.2 in \cite{wu2018note} to the setting where blocks are sampled without replacement.
\begin{align}
\mathbb{E} [ \|  \mathbf{U^\top S^\top S V} - \mathbf{U^\top V}\|_F^2]
&= \sum_{h_1 = 1}^n \sum_{h_2 = 1}^{n'} \mathbb{E} [(\mathbf{U^\top S^\top S V} - \mathbf{U^\top V})_{(h_1, h_2)}^2]
\\&= \sum_{h_1 = 1}^n \sum_{h_2 = 1}^{n'} \textrm{Var}[ (\mathbf{U^\top S^\top S V})_{(h_1, h_2)}]
\label{eq:variance_entries}
\end{align}
where the last equality follows because $\mathbb{E}[ \mathbf{U^\top S^\top S V}] = \mathbf{U^\top V}$.
We will consider an individual entry in terms of the blocks $i$.
If $p^+_i = 1$, then its contribution to the variance is 0 so we only need to consider the case where $p^+_i < 1$.
Then, using independence, we have
\begin{align}
    \textrm{Var}[ (\mathbf{U^\top S^\top S V})_{(h_1, h_2)}]
    &= \sum_{i:p^+_i < 1} \textrm{Var}[\frac1{{p^+_i}} \mathbbm{1}[i \textrm{ selected}] (\mathbf{U}_{(\Theta_i,h_1)}^\top \mathbf{V}_{(\Theta_i, h_2)})]
    \\&\leq \sum_{i:p^+_i < 1} \frac1{p^+_i} 
    \left(\mathbf{U}_{(\Theta_i,h_1)}^\top \mathbf{V}_{(\Theta_i, h_2)})\right)^2
\end{align}
so
\begin{align}
    (\ref{eq:variance_entries})
    &\leq \sum_{i: p^+_i < 1} \frac1{p_i^+} \| \mathbf{U}_{(\Theta_i,)}^\top \mathbf{V}_{(\Theta_i,)} \|_F^2
    \leq \sum_{i: p^+_i < 1} \frac1{p^+_i} \|\mathbf{U}_{(\Theta_i,)}\|_F^2 \|\mathbf{V}_{(\Theta_i,)}\|_F^2
    \\&= \sum_{i: p^+_i < 1} \frac{\sum_{k \in \Theta_i} \| \mathbf{U}_k \|_2^2}{c \sum_{k \in \Theta_i} \ell_k} \|\mathbf{V}_{(\Theta_i,)}\|_F^2
    \leq \frac1{c} \| \mathbf{V} \|_F^2 \leq \frac1{m} \| \mathbf{U} \|_F^2 \| \mathbf{V} \|_F^2
\end{align}
where the last inequality follows because $m \leq c n$ and $\| \mathbf{U} \|_F^2 = n$.

By Markov's inequality,
\begin{align}
    \Pr \left(\left\| \mathbf{U^\top S^\top S V} - \mathbf{U^\top V} \right\|_F > \epsilon \| \mathbf{U}\|_F \|\mathbf{V} \|_F \right)
    &\leq \frac{\E \left[ \left\| \mathbf{U^\top S^\top S V} - \mathbf{U^\top V} \right\|_F^2 \right]}{\epsilon^2 \| \mathbf{U}\|_F^2 \|\mathbf{V} \|_F^2}
    \leq \frac1{m \epsilon^2} \leq \delta
\end{align}
as long as $m \geq \frac1{\delta \epsilon^2}$.
\end{proof}

With the spectral and Frobenius approximation guarantees from Lemmas \ref{lemma:spectral} and \ref{lemma:frobenius}, we are ready to prove Theorem \ref{thm:leverage_shap_paired}.

\begin{proof}[Proof of Theorem \ref{thm:leverage_shap_paired}]
Observe that
\begin{align}\label{eq:to_prove_paired}
    \| \mathbf{A} \tilde{\boldsymbol{\phi}} - \mathbf{b} \|_2^2
    = \| \mathbf{A} \tilde{\boldsymbol{\phi}} - \mathbf{A} \boldsymbol{\phi} + \mathbf{A} \boldsymbol{\phi} - \mathbf{b} \|_2^2    
    = \| \mathbf{A} \tilde{\boldsymbol{\phi}} - \mathbf{A} \boldsymbol{\phi} \|_2^2 + \| \mathbf{A} \boldsymbol{\phi} - \mathbf{b} \|_2^2
\end{align}
where the second equality follows because $\mathbf{A} \boldsymbol{\phi} - \mathbf{b}$ is orthogonal to any vector in the span of $\mathbf{A}$.    
So to prove the theorem, it suffices to show that
\begin{align}
\| \mathbf{A \tilde{\boldsymbol{\phi}} - A \boldsymbol{\phi}} \|_2^2
\leq  \epsilon
\| \mathbf{A \boldsymbol{\phi} - b} \|_2^2.
\end{align}

Let $\mathbf{U} \in \mathbb{R}^{\rho \times n'}$ be an orthonormal matrix that spans the columns of $\mathbf{A}$ where $n' \leq n$.
There is some $\mathbf{y}$ such that $\mathbf{Uy} = \mathbf{A} \boldsymbol{\phi}$ and some $\tilde{\mathbf{y}}$ such that $\mathbf{U}\tilde{\mathbf{y}} = \mathbf{A} \tilde{\boldsymbol{\phi}}$.
Observe that $\| \mathbf{A \tilde{\boldsymbol{\phi}} - A \boldsymbol{\phi}} \|_2 = \| \mathbf{U}\tilde{\mathbf{y}} - \mathbf{Uy} \|_2 = \| \tilde{\mathbf{y}} - \mathbf{y} \|_2$ where the last equality follows because $\mathbf{U}^\top \mathbf{U} = \mathbf{I}$.

By the reverse triangle inequality and the submultiplicavity of the spectral norm, we have
\begin{align}
\| \tilde{\mathbf{y}} - \mathbf{y} \|_2
&\leq \| \mathbf{U^\top S^\top S U} (\tilde{\mathbf{y}} - \mathbf{y}) \|_2
+ \| \mathbf{U^\top S^\top S U} (\tilde{\mathbf{y}} - \mathbf{y})  - (\tilde{\mathbf{y}} - \mathbf{y}) \|_2
\\&\leq \| \mathbf{U^\top S^\top S U} (\tilde{\mathbf{y}} - \mathbf{y}) \|_2
+ \| \mathbf{U^\top S^\top S U} - \mathbf{I} \|_2 \| \tilde{\mathbf{y}} - \mathbf{y} \|_2.
\end{align}
Because $\mathbf{U}$ has the same leverage scores as $\mathbf{A}$ and the number of rows sampled in $\mathbf{S}$ is within a constant factor of $m$, we can apply Lemma \ref{lemma:spectral}: With $m = O(n \log \frac{n}{\delta})$, we have $\| \mathbf{U^\top S^\top S U} - \mathbf{I} \|_2 \leq \frac12$ with probability $1-\delta/2$.
So, with probability $1-\delta/2$,
\begin{align}\label{eq:diff_UTSTSU}
\| \tilde{\mathbf{y}} - \mathbf{y} \|_2
\leq 
2 \| \mathbf{U^\top S^\top S U} (\tilde{\mathbf{y}} - \mathbf{y}) \|_2.
\end{align}
Then
\begin{align}
    \| \mathbf{U^\top S^\top S U} (\tilde{\mathbf{y}} - \mathbf{y}) \|_2
    &= \left\| \mathbf{U^\top S^\top} \left( \mathbf{S U} \tilde{\mathbf{y}} - \mathbf{Sb} + \mathbf{Sb} - \mathbf{S U} \mathbf{y} \right) \right\|_2
    \\&= \left\| \mathbf{U^\top S^\top S} \left(\mathbf{U}\mathbf{y} - \mathbf{b} \right)\right\|_2   
\end{align}
where the second equality follows because $\mathbf{SU} \tilde{\mathbf{y}} -\mathbf{S b}$ is orthogonal to any vector in the span of $\mathbf{SU}$.
By similar reasoning, notice that $\mathbf{U^\top} (\mathbf{U y - b}) = \mathbf{0}$.
Then, as long as $m = \Omega(\frac{n}{\delta \epsilon})$, we have
\begin{align}
    \left\| \mathbf{U^\top S^\top S} \left(\mathbf{U}\mathbf{y} - \mathbf{b} \right)\right\|_2 
    \leq \frac{\sqrt{\epsilon}}{2 \sqrt{n}} \| \mathbf{U} \|_F \| \mathbf{Uy - b} \|_2 \label{eq:frobenius_bound}
\end{align}
with probability $1-\delta/2$ by Lemma \ref{lemma:frobenius}.
Since $\mathbf{U}$ has orthonormal columns, $\|\mathbf{U}\|_F^2 \leq n$.
Then, combining inequalities yields
\begin{align}
    \| \mathbf{A \tilde{\boldsymbol{\phi}} - A \boldsymbol{\phi}} \|_2^2 
    = \| \tilde{\mathbf{y}} - \mathbf{y} \|_2^2
    \leq 4 \| \mathbf{U^\top S^\top S U} (\tilde{\mathbf{y}} - \mathbf{y}) \|_2^2
    \leq \epsilon \| \mathbf{Uy - b} \|_2^2 = \epsilon \| \mathbf{A} \boldsymbol{\phi} - \mathbf{b} \|_2^2
\end{align}
with probability $1-\delta$.
\end{proof}

\section{Proof of Approximation Corollary}\label{appendix:corollary_proof}

We will establish several properties specific to our problem and then use these properties to prove Corollary \ref{coro:l2norm}.

\begin{lemma}[Properties of $\mathbf{A}$, $\boldsymbol{\phi}$, and $\boldsymbol{\tilde{\phi}}$]\label{lemma:properties}
The Shapley values $\boldsymbol{\phi}$, the matrix $\mathbf{A}$, and the estimated Shapley values $\boldsymbol{\tilde{\phi}}$ produced by Algorithm \ref{alg:leverage_shap} satisfy
\begin{align}
\| \mathbf{A}( \tilde{\boldsymbol{\phi}} - \boldsymbol{\phi}) \|_2^2 
= \frac1{n} \| \tilde{\boldsymbol{\phi}} - \boldsymbol{\phi} \|_2^2.
\end{align}
and 
\begin{align}
   \| \mathbf{A} \boldsymbol{\phi} \|_2^2 
   = \frac1{n} \left( \| \boldsymbol{\phi} \|_2^2 -  
   \frac{(v(\mathbf{1}) - v(\mathbf{0}))^2}{n} \right).
\end{align}

\end{lemma}
\begin{proof}
Even though $\mathbf{A}$ does not have $n$ singular values, we can exploit the special structure of $\mathbf{A}$, $\boldsymbol{\phi}$, and $\tilde{\boldsymbol{\phi}}$.
Let $\mathbf{A} = \mathbf{U \Sigma V^\top}$ be the singular value decomposition of $\mathbf{A}$.
By Lemma \ref{lemma:ATA_form}, we know $\mathbf{\Sigma} \in \mathbb{R}^{(n-1) \times (n-1)}$ is a diagonal matrix with $\frac1{\sqrt{n}}$ on the diagonal
and $\mathbf{V} \in \mathbb{R}^{n \times (n-1)}$ has $n-1$ orthonormal columns that are all orthogonal to $\mathbf{1}$.
Let $\mathbf{v}_1, \ldots, \mathbf{v}_{n-1} \in \mathbb{R}^n$ be these vectors.
Then $\mathbf{A} = \sum_{i=1}^{n-1} \frac1{\sqrt{n}} \mathbf{u}_i \mathbf{v}_i^\top$ where $\mathbf{u}_1, \ldots, \mathbf{u}_{n-1} \in \mathbb{R}^{2^n-2}$ are the $n-1$ orthonormal columns in $\mathbf{U}$.

By Lemma \ref{lemma:equivalence} and \ref{lemma:freedom}, we can write 
\begin{align}
    \boldsymbol{\phi} = \sum_{i=1}^{n-1} \mathbf{v}_i s_i + \mathbf{1} \frac{v(\mathbf{1}) - v(\mathbf{0})}{n}    
    \qquad
    \textnormal{and}
    \qquad
    \tilde{\boldsymbol{\phi}} = \sum_{i=1}^{n-1} \mathbf{v}_i \tilde{s}_i + \mathbf{1} \frac{v(\mathbf{1}) - v(\mathbf{0})}{n}    
\end{align}
for some $s_i$ and $\tilde{s}_i$ where $i \in [n-1]$.
Then
    $(\tilde{\boldsymbol{\phi}} - \boldsymbol{\phi})
    = \sum_{i=1}^{n-1} \mathbf{v}_i (\tilde{s}_i - s_i)$
and
\begin{align}
    \mathbf{A}( \tilde{\boldsymbol{\phi}} - \boldsymbol{\phi})
    = \sum_{i=1}^{n-1} \frac1{\sqrt{n}} \mathbf{u}_i \mathbf{v}_i^\top 
    \sum_{j=1}^{n-1}\mathbf{v}_j (\tilde{s}_j - s_j)
    = \sum_{i=1}^{n-1} \frac1{\sqrt{n}} \mathbf{u}_i (\tilde{s}_i - s_i)
\end{align}
so 
$\| \mathbf{A}( \tilde{\boldsymbol{\phi}} - \boldsymbol{\phi}) \|_2^2 
= \frac1{n} \| \tilde{\boldsymbol{\phi}} - \boldsymbol{\phi} \|_2^2.$

Similarly, we can write
\begin{align}
    \mathbf{A}\boldsymbol{\phi}
    = \sum_{i=1}^{n-1} \frac1{\sqrt{n}} \mathbf{u}_i \mathbf{v}_i^\top 
    \left(\sum_{j=1}^{n-1}\mathbf{v}_j s_j + \mathbf{1} \frac{v(\mathbf{1}) - v(\mathbf{0})}{n} \right)
    = \sum_{i=1}^{n-1} \frac1{\sqrt{n}} \mathbf{u}_i s_i
\end{align}
where the second equality follows because $\mathbf{v}_1, \ldots, \mathbf{v}_{n-1}$ are orthogonal to $\mathbf{1}$.
Then $\| \mathbf{A} \boldsymbol{\phi} \|_2^2 
= \frac1{n} \left( \| \boldsymbol{\phi} \|_2^2 -  
\frac{(v(\mathbf{1}) - v(\mathbf{0}))^2}{n} \right).$

\end{proof}

\corollarynorm*

\begin{proof}[Proof of Corollary \ref{coro:l2norm}]
We have
\begin{align}
    \| \mathbf{A} \tilde{\boldsymbol{\phi}} - \mathbf{b} \|_2^2
    = \| \mathbf{A} \tilde{\boldsymbol{\phi}} - \mathbf{A} \boldsymbol{\phi} + \mathbf{A} \boldsymbol{\phi} - \mathbf{b} \|_2^2    
    = \| \mathbf{A} \tilde{\boldsymbol{\phi}} - \mathbf{A} \boldsymbol{\phi} \|_2^2 + \| \mathbf{A} \boldsymbol{\phi} - \mathbf{b} \|_2^2
\end{align}
where the second equality follows because $\mathbf{A} \boldsymbol{\phi} - \mathbf{b}$ is orthogonal to any vector in the span of $\mathbf{A}$.

Then, by the assumption, we have
\begin{align}\label{eq:to_prove}
    \| \mathbf{A} \tilde{\boldsymbol{\phi}} - \mathbf{A} \boldsymbol{\phi} \|_2^2 \leq \epsilon \| \mathbf{A} \boldsymbol{\phi} - \mathbf{b} \|_2^2.
\end{align}

By the definition of $\gamma = \frac{\| \mathbf{A} \boldsymbol{\phi} - \mathbf{b} \|_2^2}{\| \mathbf{A} \boldsymbol{\phi} \|_2^2}$ and Lemma \ref{lemma:properties},
\begin{align}\label{eq:l2_upperbound}
    \epsilon \| \mathbf{A} \boldsymbol{\phi} - \mathbf{b} \|_2^2
    = \epsilon \gamma \| \mathbf{A} \boldsymbol{\phi} \|_2^2
    = \epsilon \gamma \frac1{n} \left(\| \boldsymbol{\phi} \|_2^2 - \frac{(v(\mathbf{1}) - v(\mathbf{0}))^2}{n}\right).
\end{align}
Finally, by Lemma \ref{lemma:properties} along with Equations \ref{eq:to_prove} and \ref{eq:l2_upperbound}, we have
\begin{align}
\frac1{n} \| \tilde{\boldsymbol{\phi}} - \boldsymbol{\phi} \|_2^2 =
\| \mathbf{A}( \tilde{\boldsymbol{\phi}} - \boldsymbol{\phi}) \|_2^2 
\leq \epsilon \| \mathbf{A} \boldsymbol{\phi} - \mathbf{b} \|_2^2
\leq \frac1{n} \epsilon \gamma \| \boldsymbol{\phi} \|_2^2.
\end{align}
The corollary statement then follows after multiplying both sides by $n$.
\end{proof}

\section{Computational Hardness}\label{app:hardness}

When the Shapley value problem is viewed as an optimization problem, we provide a constant factor approximation in polynomial time (see e.g., Theorem \ref{thm:main}).

On the other hand, suppose we wanted to, e.g., ensure that $\frac{1}{C} \phi_i \leq \tilde{\phi}_i \leq C \phi_i$ for all $i$. We claim that obtaining this goal for any constant $C$ is NP-hard in many settings. In particular, to ask about computational hardness, we also need to specify the input to the algorithm: i.e., is $v$ given as a polynomial size circuit, a polynomial size formula, or as a black-box, unit cost oracle? In all of these settings the problem is NP-hard to approximate. In particular, consider the case when $v$ either evaluates to $0$ for all sets, or evaluates to $0$ for all sets except for a single set $S$. In the first case, the Shapley values will all equal $0$, whereas in the second case, the Shapley values for indices in $S$ are non-zero. As such, to obtain a multiplicative approximation, we need to determine if there is some set $S$ for which $v$ does not evaluate to 0. When $v$ is a black-box, finding such a set clearly requires $\Omega(2^n)$ time (we can only enumerate all possible inputs). However, it is still hard when $v$ is given other forms. For example, if $v$ is a given as a circuit, then determining if there is an S for which $v(S)\neq 0$ is equivalent to the NP-complete circuit SAT problem.

This argument makes it clear why asking for a multiplicative approximation is not reasonable, and why instead we might care about an approximation in objective value. This is the case for many computational problems: for example, in the class of NP-hard optimization problems like set cover, we do not typically ask to approximate the optimal solution itself, but instead to find another solution with nearly the same objective value as the optimal solution (whether or not it is similar to that optimal solution or not).

\section{Noisy Access to the Set Function}\label{app:noise}

In this section, we explore how the regression-based estimators perform given noisy access to the set function $v$.
This setting is particularly relevant for Shapley values in explainable AI, since the set function may be an expectation that is approximated in practice.
Figure \ref{fig:noise-l2} shows plots performance as noise is added to the set functions:
Instead of observing $v(S)$, the estimators observe $v(S) + \zeta$ where $\zeta \sim \mathcal{N}(0, \sigma^2)$ is normally distributed with standard deviation $\sigma$ (we set $\sigma = 5 \times 10^{-4}, 10^{-3}, 5 \times 10^{-3}, 10^{-2}, 5 \times 10^{-2}, 10^{-1}, 0.5, 1$, with $m = 10n$ samples).

\begin{figure}[h]
    \centering
    \includegraphics[width=\linewidth]{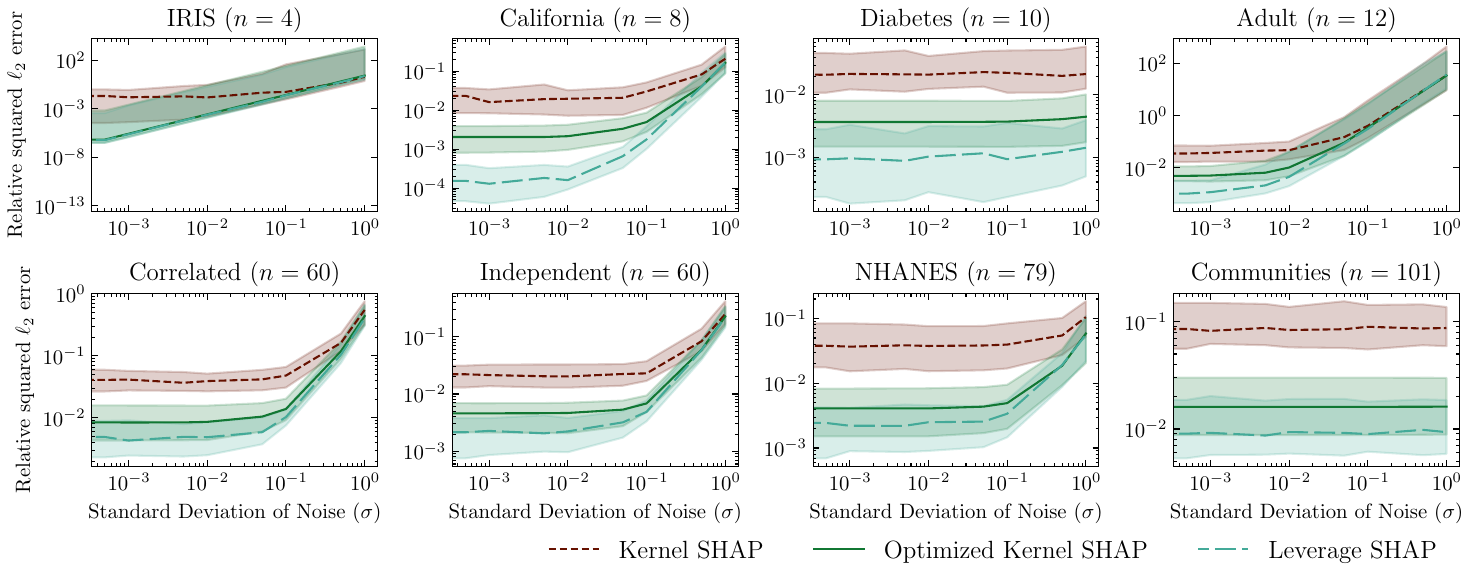}
    \caption{The relative squared $\ell_2$ error between the estimated and ground truth Shapley values as a function of noise in the set function. Leverage SHAP gives the best performance at low and moderate noise ($\sigma \leq 0.1$) on every dataset, and at every noise level on Diabetes, Correlated, Independent, and Communities. At the highest noise levels ($\sigma \geq 0.5$), where the noise dominates the error, Leverage SHAP and Optimized Kernel SHAP are close, and Optimized Kernel SHAP pulls slightly ahead in some settings on California, Adult, and NHANES. On IRIS, where $m \geq 2^n$, the noise alone determines the error.}
    \label{fig:noise-l2}
    \vspace{-1em}
\end{figure}

\section{Error Relative to Optimized Kernel SHAP}\label{sec:error2kshap}

\begin{table}[h]
\centering
{ 
\input{tables/error_ratio}
}
\caption{The table reports the average error ratio over 100 random runs between Leverage SHAP and Optimized Kernel SHAP. For example, on the \textit{Independent} dataset with $m = 5n$ samples, Leverage SHAP achieved estimates with $49\%$ of the error of Optimized Kernel SHAP. A ratio of $1$ indicates that $m \geq 2^n$, and both algorithms recovered estimates accurate to machine precision. On average (excluding entries where the ratio is $1$), Leverage SHAP estimates have $50.0\%$ of the error of Optimized Kernel SHAP. In other words, Leverage SHAP achieves roughly a $50\%$ reduction in error compared to Optimized Kernel SHAP.}
\label{tab:error2kshap}
\end{table}

Table \ref{tab:error2kshap} presents the average error ratios between Leverage SHAP and Optimized Kernel SHAP across various datasets and sample sizes, computed over 100 independent random runs. Each entry reports the ratio of the mean relative squared $\ell_2$ error of Leverage SHAP to that of Optimized Kernel SHAP. For example, on the \textit{Independent} dataset with $m = 5n$ samples, Leverage SHAP achieved an error that was $49\%$ of the error of Optimized Kernel SHAP. A reported ratio of $1$ corresponds to cases where $m \geq 2^n$, in which both methods recovered estimates accurate to machine precision. Averaging over all non-unit entries, Leverage SHAP achieved $50.0\%$ of the error of Optimized Kernel SHAP, corresponding to an approximate $50\%$ reduction in error.

\clearpage

\section{Exploring $\gamma$}\label{appendix:gamma}

The problem-specific term $\gamma = \| \mathbf{A} \boldsymbol{\phi} - \mathbf{b} \|_2^2 / \| \mathbf{A} \boldsymbol{\phi} \|_2^2$ appears in the $\ell_2$-norm guarantee of Corollary \ref{coro:l2norm}.
In this section, we explore the distribution of $\gamma$, and the performance by $\gamma$.
Because exactly computing $\gamma$ requires evaluating $\mathbf{b}$, we focus on datasets where $n \leq 16$.

Table \ref{tab:gamma} shows summary statistics of $\gamma$ over 100 runs (the fitted model is the same in every run; the randomness comes from the choice of explicand). On every dataset, the third quartile of $\gamma$ is below 1. 

\begin{table}[h]
    \centering
    \caption{Summary statistic of $\gamma$ over 100 runs.}
    \input{tables/gamma_distribution}
    \label{tab:gamma}
\end{table}

In Figure \ref{fig:gamma-l2}, we explore experimentally whether $\gamma$ is an artifact of our analysis or a necessary parameter.
We perform the experiment by explicitly building the matrix $\mathbf{A}$ (only possible for small $n$) and then choosing $\mathbf{b}$ as a linear combination of a vector in the column span of $\mathbf{A}$ and a vector not in the column span of $\mathbf{A}$.
Then, we back out the corresponding set function $v$ using the definition of $\mathbf{b}$.
We induce different values of $\gamma$ as we vary how much of $\mathbf{b}$ is a vector in the column span of $\mathbf{A}$.
Because all three of the algorithms we consider perform worse as $\gamma$ increases, the experiment suggests that $\gamma$ is not an artifact of our analysis.
The same trend appears for the ablated estimators, as shown in Figure \ref{fig:ablation_gamma-l2}.
As our analysis suggests, $\gamma$ does not appear to be a relevant factor for the objective error, as shown in Figure \ref{fig:ablation_gamma-objective}.

\begin{figure}[h!]
    \centering
    \includegraphics[width=\linewidth]{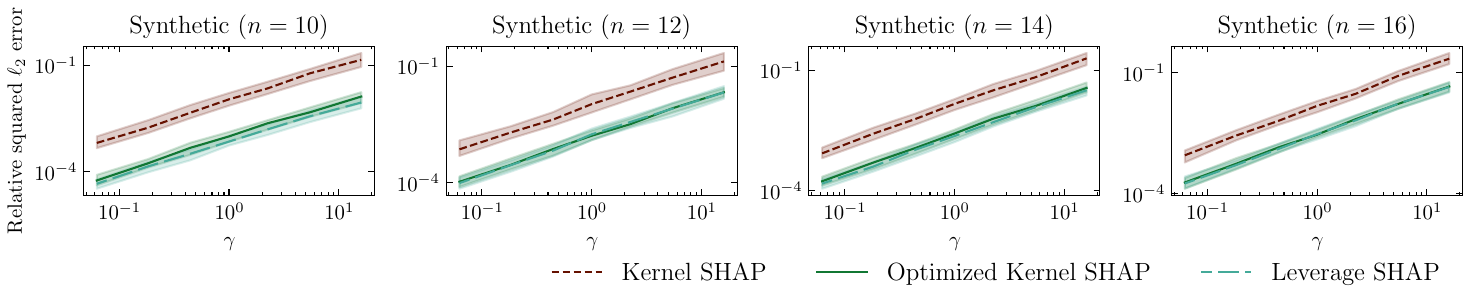}
    \caption{The relative squared $\ell_2$ error as a function of $\gamma$, the problem-specific parameter in Corollary \ref{coro:l2norm}. The lines indicate the median and the shaded regions encompass the first and third quartile over 100 runs. All three of the main algorithms we consider have higher relative squared $\ell_2$ error as $\gamma$ grows, suggesting that $\gamma$ is not an artifact of our analysis. Leverage SHAP has lower median error at $n=10$ and $n=14$, while Leverage SHAP and Optimized Kernel SHAP are closer at $n=12$ and $n=16$.}
    \label{fig:gamma-l2}
\end{figure}

\begin{figure}[h!]
    \centering
    \includegraphics[width=\linewidth]{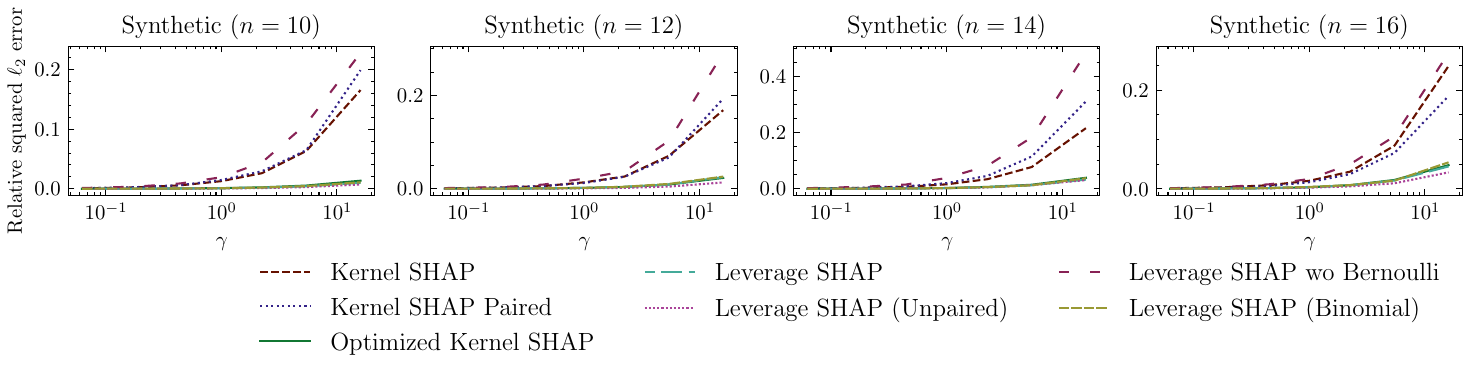}
    \caption{The relative squared $\ell_2$ error as a function of $\gamma$, the problem-specific parameter in Corollary \ref{coro:l2norm}. The lines indicate the mean over 100 runs. All algorithms we consider have higher relative squared $\ell_2$ error as $\gamma$ grows, suggesting that $\gamma$ is not an artifact of our analysis.}
    \label{fig:ablation_gamma-l2}
\end{figure}

\begin{figure}[h!]
    \centering
    \includegraphics[width=\linewidth]{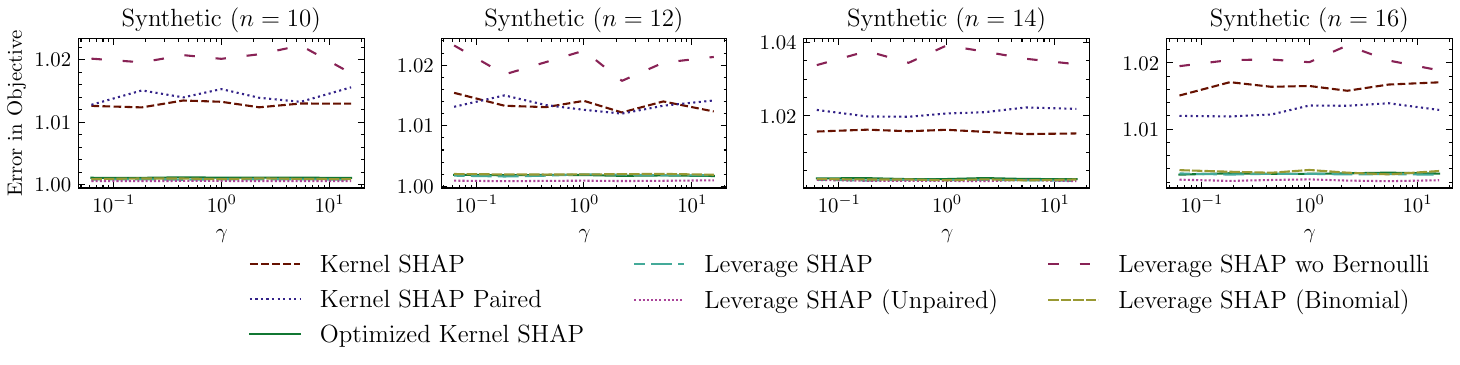}
    \caption{The linear objective error as a function of $\gamma$, the problem-specific parameter in Corollary \ref{coro:l2norm}. The lines indicate the mean over 100 runs. The error of every algorithm we consider stays flat as $\gamma$ grows, suggesting that $\gamma$ is not a relevant parameter for the linear objective error. This finding aligns with Theorem \ref{thm:main}, which has no dependence on $\gamma$.}
    \label{fig:ablation_gamma-objective}
\end{figure}

\section{Additional Experiments with Relative Squared $\ell_2$ Error}\label{appendix:ablation}

Table \ref{tab:additional} and Figures \ref{fig:ablation_size-l2} and \ref{fig:ablation_noise-l2} evaluate eight estimators.
\textit{Kernel SHAP} is our implementation of the regression estimator with rows sampled with replacement according to the Kernel SHAP weights and without pairing; \textit{Kernel SHAP Paired} adds paired sampling.
\textit{Optimized Kernel SHAP} and \textit{Permutation SHAP} are the implementations in the SHAP library \citep{lundberg2017unified}; the former is the Kernel SHAP baseline of Section \ref{sec:experiments}.
\textit{Leverage SHAP} is the released implementation of Algorithm \ref{alg:leverage_shap} (paired sampling, sampling without replacement, and the deterministic per-size counts described in Section \ref{sec:our_algorithm}); \textit{Leverage SHAP (Binomial)} uses randomized per-size counts with the corrected middle-layer pair draw described below; \textit{Leverage SHAP wo Bernoulli} samples from the leverage score distribution with replacement (keeping paired sampling); and \textit{Leverage SHAP (Unpaired)} drops paired sampling.
The Binomial results were recomputed with 100 seeded repetitions per setting (14,000 runs in total across the real and synthetic experiments). For even $n$, the middle-layer pair count is drawn from $\operatorname{Binomial}(\binom{n}{n/2}/2,p)$, as in Algorithm \ref{alg:bernoulli_sampling}. The nominal budget $m$ is a target: the actual evaluation count is random and is capped by full enumeration at $2^n$. Across these settings, the mean realized evaluation count is within $2.7\%$ of $\min(m,2^n)$. The other estimators retain their previous runs; the same model settings and explicand seeds are used, but the new Binomial sampling seeds differ from the historical runs. When the number of combinations exceeds the implementation's integer range, it still replaces a Binomial count by its rounded-down expectation and draws subsets with replacement. This approximation affects NHANES and Communities, so the independent-block theorem does not apply verbatim to those implementations.

Leverage SHAP, its Binomial variant, and Leverage SHAP without Bernoulli sampling are within $20\%$ of one another in every cell of the large-$n$ datasets (and within $8\%$ in the mean on Correlated, Independent, and Communities), and, excluding the ties on IRIS, they take the top three places in all but two cells of Table \ref{tab:additional} (Kernel SHAP Paired takes third place for the first quartile on California, and Permutation SHAP takes first place for the first quartile on NHANES).
Apart from the ties on IRIS and the NHANES first quartile, Leverage SHAP outperforms Permutation SHAP and both Kernel SHAP implementations in every cell.

\begin{table}[H]
    \centering
    \input{tables/ablation_shap_error}
    \caption{Relative squared $\ell_2$ error by dataset and estimator, reported in summary statistics from 100 random runs. The Binomial entry uses the corrected sampler and fresh runs described in this appendix.}\label{tab:additional}
\end{table}

Figures \ref{fig:ablation_size-l2}, and \ref{fig:ablation_noise-l2} show additional experiments where the error is measured with the relative squared $\ell_2$ error. We find that Leverage SHAP and its Binomial variant, which trade places from setting to setting, give the best performance.

\begin{figure}[h!]
    \centering
    \includegraphics[width=\linewidth]{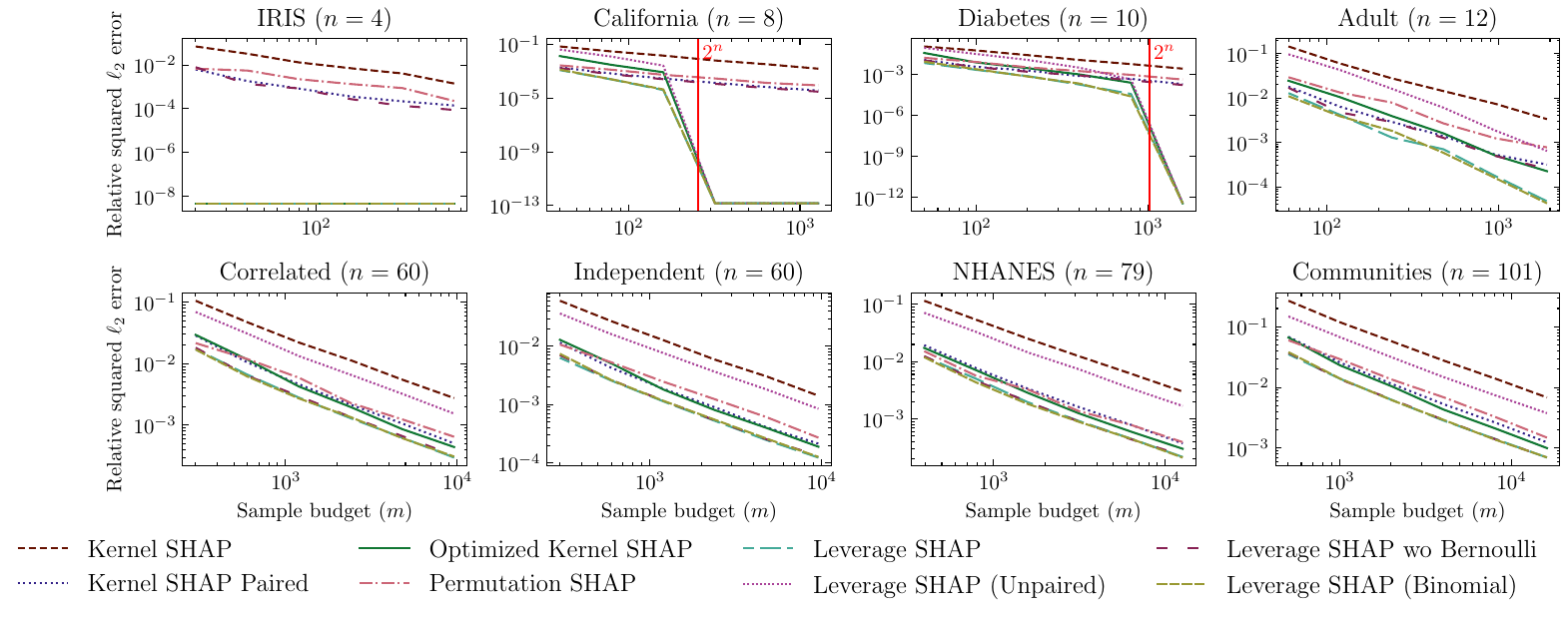}
    \caption{The relative squared $\ell_2$ error between the estimated Shapley values and ground truth Shapley values as a function of sample size. The lines report the mean error over 100 runs. Leverage SHAP and its Binomial variant give the best performance in nearly every setting with $m < 2^n$. Leverage SHAP without Bernoulli sampling is within $10\%$ of them for large $n$ once $m \geq 10n$ (at $m = 5n$ the gap is only slightly larger, $4$--$12\%$), but between $1.1$ and $15.2$ times as bad for small $n$, with the gap growing as $m$ approaches $2^n$ and sampling without replacement matters more. Kernel SHAP Paired and Optimized Kernel SHAP typically come next, followed by Permutation SHAP and then Leverage SHAP (Unpaired), although Permutation SHAP beats Optimized Kernel SHAP on California and on most of NHANES; Kernel SHAP is worst in every setting.}
    \label{fig:ablation_size-l2}
\end{figure}

\begin{figure}[h!]
    \centering
    \includegraphics[width=\linewidth]{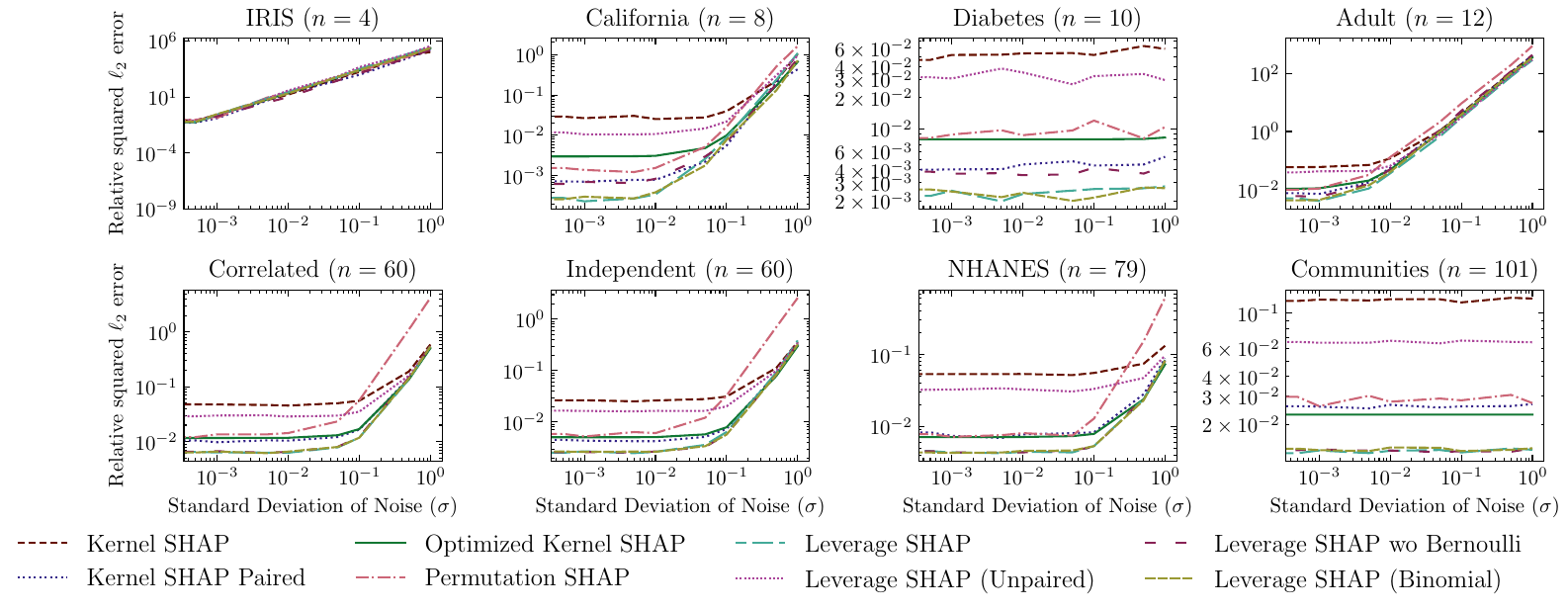}
    \caption{The relative squared $\ell_2$ error between the estimated and ground truth Shapley values as a function of noise in the set function. For $\sigma \leq 0.1$, Leverage SHAP, its Binomial variant, and (on the large-$n$ datasets) Leverage SHAP without Bernoulli sampling trade places from setting to setting and give the best performance in every setting with $m < 2^n$ except California at $\sigma = 0.1$, where Kernel SHAP Paired is best. At the highest noise levels ($\sigma \geq 0.5$), Optimized Kernel SHAP is best on Correlated and Independent, and on NHANES at $\sigma = 1$; the Binomial variant is best on NHANES at $\sigma = 0.5$. Kernel SHAP Paired is best on California at $\sigma = 1$. On IRIS, the methods that enumerate all $2^n$ coalitions have error dominated by noise.}
    \label{fig:ablation_noise-l2}
\end{figure}

\section{Ablation Experiments with Objective Error }\label{appendix:objective}

Figures \ref{fig:ablation_size-objective} and \ref{fig:ablation_noise-objective} show ablation experiments where the error is measured with the objective value naturally suggested by the regression formulation. In the sample size experiment, Leverage SHAP or its Binomial variant gives the best performance in nearly every setting, although once $m \geq 10n$ every estimator except Leverage SHAP (Unpaired) and Kernel SHAP is within a few percent of the optimum; in the noise experiment there is no clear winner.

\begin{figure}[h!]
    \centering
    \includegraphics[width=\linewidth]{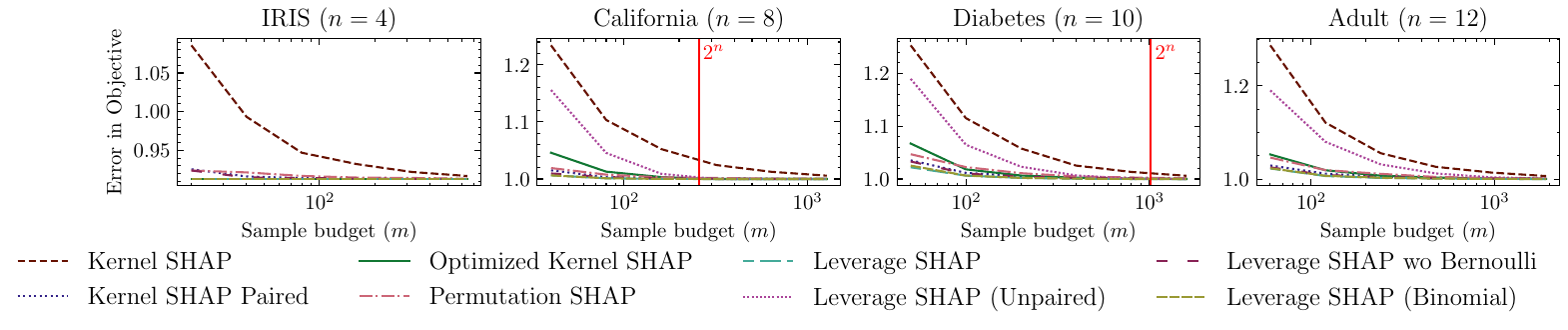}
    \caption{Linear objective error by sample size. The lines indicate the mean error over 100 runs. Leverage SHAP and its Binomial variant quickly achieve the lowest error in nearly all settings. (For small $n$, numerical instability can lead to error below 1.)}
    \label{fig:ablation_size-objective}
\end{figure}

\begin{figure}[h!]
    \centering
    \includegraphics[width=\linewidth]{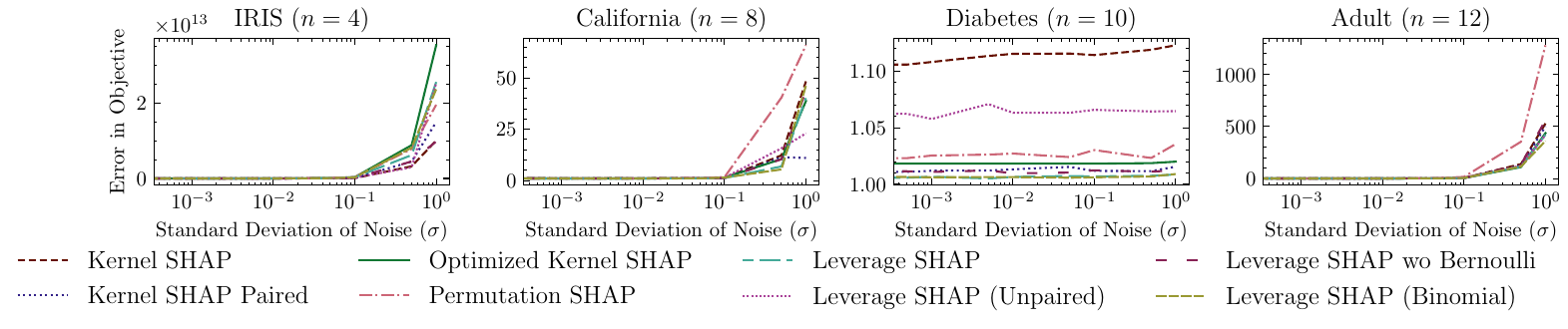}
    \caption{Linear objective error by standard deviation of noise. The noise grows large as $\sigma$ approaches one for all algorithms. The plots indicate no clear winner across all four datasets.}
    \label{fig:ablation_noise-objective}
\end{figure}

\section{Constrained and Unconstrained Equivalence}\label{appendix:equivalence}

\equivalence*

We provide a cleaner, direct proof using the unconstrained problem and the exact characterization of $\mathbf{Z^\top Z}$.

\begin{proof}[Proof of Lemma \ref{lemma:equivalence}]
We will show that
\begin{align}\label{eq:to_show}
    \boldsymbol{\phi} = \argmin_{\mathbf{x}} \| \mathbf{Ax - b} \|_2^2 + \mathbf{1} \frac{v(\mathbf{1}) - v(\mathbf{0})}{n}
\end{align}
where
$\mathbf{A} = \mathbf{Z P}$ and $\mathbf{b} =\mathbf{y} - \mathbf{Z 1} \frac{v(\mathbf{1}) - v(\mathbf{0})}{n}.$
Then Lemma \ref{lemma:equivalence} follows from Lemma \ref{lemma:freedom}.

We know
\begin{align}
    \argmin_{\mathbf{x}} \| \mathbf{Ax - b} \|_2^2
    = (\mathbf{A^\top A})^{+} \mathbf{A^\top b}.
\end{align}
By Lemma \ref{lemma:ATA_form}, we have $(\mathbf{A^\top A})^{+} = (\mathbf{P^\top Z^\top Z P})^{+} = n \mathbf{I} - \mathbf{1 1^\top} = n\mathbf{P}$.
Then
\begin{align}
    (\mathbf{A^\top A})^{+} \mathbf{A}^\top = 
    n \mathbf{P} \mathbf{P}^\top \mathbf{Z}^\top
    = n \mathbf{P} \mathbf{Z}^\top.
\end{align}
Let $\mathbf{z} \in \{0,1\}^n$ with $0 < \| \mathbf{z} \|_1 < n$.
It follows that
\begin{align*}
    (\mathbf{A}^\top \mathbf{A})^{+} \mathbf{A}^\top \mathbf{b}
    = n \mathbf{P} \left(
        \mathbf{Z}^\top 
        \begin{bmatrix}
            \vdots \\ \sqrt{w(\| \mathbf{z}\|_1)} \left(v(\mathbf{z}) - v(\mathbf{0})\right) \\ \vdots
        \end{bmatrix}
        - \mathbf{Z}^\top
        \begin{bmatrix}
            \vdots \\ \sqrt{w(\| \mathbf{z}\|_1)} \| \mathbf{z}\|_1 \frac{v(\mathbf{1})-v(\mathbf{0})}{n} \\ \vdots
        \end{bmatrix}
    \right).
    \end{align*}
By symmetry, the last term is a scaling of the all ones vector so it gets projected off by $\mathbf{P}$.
Recall $w(s) = \frac{(s-1)!(n-s-1)!}{n!}$.
The weights are carefully designed so that $w(s) s = w(s +1) (n - s -1)$.
Then, for $i \in [n]$, we have the $i$th entry of $(\mathbf{A^\top A})^{+} \mathbf{A}^\top \mathbf{b}$ given by 
\begin{align}
     &\sum_{\mathbf{z}: 0 < \| \mathbf{z} \|_1 < n}
     w(\| \mathbf{z} \|_1)
     (n z_i - \| \mathbf{z} \|_1)
     \left(v(\mathbf{z}) - v(\mathbf{0})\right)
     \nonumber \\ &=
     \frac1{n} \sum_{S: \{i \} \subseteq S \subset [n]} \binom{n-1}{|S|-1}^{-1} v(S)
        - \frac1{n} \sum_{S: \emptyset \subset S \subseteq [n] \setminus \{i\} } \binom{n-1}{|S|}^{-1} v(S)
        \\
        & + \frac{v([n]) - v(\emptyset)}{n} - \frac{v([n]) - v(\emptyset)}{n}
        \\&= \frac1{n} \sum_{S: \{i \} \subseteq S \subseteq [n]} \binom{n-1}{|S|-1}^{-1} v(S)
        - \frac1{n} \sum_{S: \emptyset \subseteq S \subseteq [n] \setminus \{i\} } \binom{n-1}{|S|}^{-1} v(S)
        \\& - \frac{v([n]) - v(\emptyset)}{n}
        \\&
        = \phi_i - \frac{v([n]) - v(\emptyset)}{n}
     \label{eq:set_conversion}
\end{align}
where the first equality follows by converting from binary vectors to sets.
Then Equation \ref{eq:to_show} follows.

\end{proof}

\section{Efficiently Sampling Without Replacement}\label{appendix:bernoulli_code}

Since there are an exponential number of rows,
the naive method of considering each row independently will not work.
Our approach is as follows:
For each set size $s \in [\lfloor n/2 \rfloor]$, we sample the number of samples that will be chosen.
For $s<n/2$, this is distributed as a Binomial distribution with $\binom{n}{s}$ trials and probability $\min(1, 2c \binom{n}{s}^{-1})$.
Then, because the probabilities are the same for all subsets of size $s$, we exploit this symmetry and sample random indices uniformly from $[\binom{n}{s}]$ without replacement.
We index all subsets of size $s$ and then construct the subsets corresponding to the randomly sampled indices without enumerating all subsets.
The Bernoulli sampling code appears in Algorithm \ref{alg:bernoulli_sampling} and the subset construction code appears in Algorithm \ref{alg:combo}, both deferred to Appendix \ref{appendix:bernoulli_code}.

Like before, there is one special case when $n$ is even and there is a middle set size $s = \lfloor n/2 \rfloor$.
Here, the paired samples $\mathbf{z}$ and $\bar{\mathbf{z}}$ have the same size $s$ so we need to partition the set of all subsets of size $s$ so we do not risk sampling the same pair twice.
We accomplish this by sampling $\mathbf{z}$ from all subsets of $n-1$ items with size $s-1$ then appending $\mathbf{z}_n=1$ and computing the complement $\bar{\mathbf{z}}$. We therefore draw the number of middle-layer pairs from $\operatorname{Binomial}(\binom{n-1}{s-1},\min(1,2c\binom{n}{s}^{-1}))$, rather than halving a Binomial draw over all middle-layer sets.

\begin{algorithm}
\caption{\texttt{BernoulliSample}$(n, c)$: Efficient Bernoulli Sampling}\label{alg:bernoulli_sampling}
\begin{algorithmic}[1]
\Input $n$: number of players, $c$: probability scaling
\Output $\mathcal{Z}'$: random sample of $\mathcal{Z}$ that independently contains pair $(\mathbf{z}, \bar{\mathbf{z}})$ w.p. $\min(1, 2c \cdot \ell_z)$
\State $\mathcal{Z}' \gets \emptyset$
\For{set size $s \in \{1,\ldots,\lfloor \frac{n}{2} \rfloor \}$}
\State $\textrm{isMiddle} \gets (2 | n) \land (s = \lfloor \frac{n}{2} \rfloor)$ \Comment{Special handling of middle set size for paired samples}
\State $N_s\gets\binom{n-1}{s-1}$ if isMiddle, else $N_s\gets\binom{n}{s}$
\State $m_s \sim \textrm{Binomial}\left(N_s, \min(1, 2c \cdot \binom{n}{s}^{-1})\right)$\label{line:binomial_sample}
\State randomIndices $\gets$ $m_s$ uniform random integers drawn without replacement from $\{0,\ldots,N_s-1\}$
\For{$i \in $ randomIndices}
\State $\mathbf{z} \gets$ \texttt{Combo}$(n, s, i)$ \Comment{$i$th combination of $n$ items with size $s$ (Algorithm \ref{alg:combo})}
\If{isMiddle} \Comment{Partition the middle size by setting $\mathbf{z}_n = 1$}
\State $\mathbf{z} \gets$\texttt{Combo}$(n-1, s-1, i)$ \Comment{$i$th combination of $n-1$ items with size $s-1$}
\State Append $1$ to $\mathbf{z}$
\EndIf
\State Add $(\mathbf{z}, \bar{\mathbf{z}})$ to $\mathcal{Z}'$
\EndFor
\EndFor
\State \textbf{return} $\mathcal{Z}'$
\end{algorithmic}
\end{algorithm}

\begin{algorithm}
\caption{\texttt{Combo}$(n, s, i)$: Compute the $i$th Combination in Lexicographic Order}\label{alg:combo}
\begin{algorithmic}[1]
\State \textbf{Input:} $n$: total number of elements, $s$: subset size, $i$: index of subset
\State \textbf{Output:} $\mathbf{z}$: the $i$th combination (lexicographically) in binary form

\State $\mathbf{z} \gets \mathbf{0}_{n}$, $k \gets s$, start $\gets 1$ 

\For{idx $ \in \{1, \ldots, s\}$}
    \For{$j \in \{\textrm{start}, \ldots, n\}$}
        \State count $\gets \binom{n - j}{k - 1}$ \Comment{Combinations possible with remaining elements if $j$ is added}
        \If{$i < \textrm{count}$}
            \State $z_j \gets 1$ \Comment{Add $j$}
            \State $k \gets k - 1$ \Comment{Decrease the number of elements left to choose}
            \State $\textrm{start} \gets j + 1$ \Comment{Update starting index to ensure lexicographic order}
            \State \textbf{break}
        \Else
            \State $i \gets i - \textrm{count}$ \Comment{Skip this batch of combinations}
        \EndIf
    \EndFor
\EndFor
\State \Return $\mathbf{z}$
\end{algorithmic}
\end{algorithm}

\begin{lemma}[Time Complexity]\label{lemma:time}
    Let $m'$ be the number of samples used in Algorithm \ref{alg:leverage_shap} and define $T_{m'}$ as the time complexity of evaluating $v$ in parallel on $m'$ inputs.
    Algorithm \ref{alg:leverage_shap} runs in time $O(m'n^2 + n^3 + T_{m'})$.
\end{lemma}

For most settings, we expect $T_{m'}$ to dominate the time complexity.
For example, consider a shallow fully connected neural network with inner dimension just $n$.
Even evaluating the forward pass on $m'$ inputs for a single layer of this network takes time $O(m'n^2)$.

\begin{proof}[Proof of Lemma \ref{lemma:time}]
    \texttt{BernoulliSample} takes $O(mn^2)$ time:
    For each pair of samples, we call \texttt{Combo} at most twice which takes time at most $O(n^2)$ since there are only two loops, each over at most $n$ items.
    For simplicity of presentation, the algorithm calls $v$ three times: once with input $\mathbf{1}$, once with input $\mathbf{0}$, and once with $m'-2$ inputs.
    However, by concatenating the inputs, the algorithm can be easily modified to call $v$ once on $m'$ inputs for a time complexity of $T_{m'}$.
    Computing the optimal subsampled solution $\tilde{\boldsymbol{\phi}}^\perp$ takes $O(n^3)$ time to compute a factorization of a $n \times n$ matrix and $O(m'n^2)$ time to compute the $n \times n$ matrix.
    Thus the additional runtime is $O(m'n^2+n^3)$. Since $\mathbb{E}[m']=m$ and the target budget satisfies $m\geq n$, the expected additional runtime is $O(mn^2)$.
\end{proof}

\section{Constrained Regression}\label{appendix:constrained_solution}

\freedomlemma*

\begin{proof}[Proof of Lemma \ref{lemma:freedom}]
We will decompose $\mathbf{x}$ into one vector that is orthogonal to $\mathbf{1}$ and another vector that is a scaling of $\mathbf{1}.$ That is, $\mathbf{x} = \mathbf{x}' + c\mathbf{1}$ where $\langle \mathbf{x'},\mathbf{1} \rangle = 0$ for some $c$.
In order to satisfy the constraint, it follows that $c = \frac{v(\mathbf{1}) - v(\mathbf{0})}{n}$.
Then 
\begin{align}
    \argmin_{\mathbf{x}: \langle \mathbf{x}, \mathbf{1} \rangle= v(\mathbf{1}) - v(\mathbf{0})} \| \mathbf{Zx} - \mathbf{y} \|_2^2
    &= \argmin_{\mathbf{x'} + c\mathbf{1}: \langle \mathbf{x'}, \mathbf{1} \rangle = 0}
    \| \mathbf{Z}(\mathbf{x'} + c\mathbf{1}) - \mathbf{y} \|_2^2
    \\&= \argmin_{\mathbf{x'} : \langle \mathbf{x'}, \mathbf{1} \rangle = 0}
    \| \mathbf{Zx'} - (\mathbf{y} - c \mathbf{Z} \mathbf{1}) \|_2^2 + c \mathbf{1}
    \\&= \argmin_{\mathbf{x}'} \| \mathbf{ZPx'} - (\mathbf{y} - c \mathbf{Z} \mathbf{1}) \|_2^2 + c \mathbf{1}
    \label{eq:decomposed}
\end{align}
We used $\mathbf{P}$ to project off any component in the direction of $\mathbf{1}$ and thereby remove the constraint.
Plugging in the value of $c$ with the definitions of $\mathbf{A}$ and $\mathbf{b}$ yields the first equation.
The second equation follows by a similar argument.
\end{proof}

\paragraph{Traditional Solution Used in Kernel SHAP}

Consider the problem
\begin{align}
    \mathbf{x^*} = \argmin_{\mathbf{x}: \langle \mathbf{x}, \mathbf{1} \rangle = v(\mathbf{1}) - v(\mathbf{0})} 
    \| \mathbf{Zx} - \mathbf{y} \|_2^2.
\end{align}

For Leverage SHAP, we reformulate the problem to an unconstrained regression problem with Lemma \ref{lemma:freedom} and solve it using standard least squares.

In the standard Kernel SHAP implementation, the constrained problem is solved directly using the method of Lagrange multipliers.
We put the analysis in here for completeness.

The Lagrangian function with respect to $\mathbf{x}$ and multiplier $\lambda \in \mathbb{R}$ is
\begin{align}
    \mathcal{L}(\mathbf{x}, \lambda)
    &= \| \mathbf{Zx} - \mathbf{y} \|_2^2 + \lambda (\langle \mathbf{x}, \mathbf{1} \rangle - v(\mathbf{1}) + v(\mathbf{0}))
\end{align}

At the optimal solution $\mathbf{x}^*$ and $\lambda^*$, the gradients with respect to $\mathbf{x}$ and $\lambda$ are 0.
This observation implies the following two equations:
\begin{align}
    &\nabla_{\mathbf{x}} \mathcal{L}(\mathbf{x}^*,\lambda^*) 
    = 2 \mathbf{Z}^\top (\mathbf{Z} \mathbf{x}^* - \mathbf{y}) + \lambda^* \mathbf{1} = 0
    \quad \implies \quad \mathbf{x}^* 
    = (\mathbf{Z^\top Z})^{-1} \left(\mathbf{Z}^\top \mathbf{y} - \frac{\lambda^*}{2}\mathbf{1}\right)
    \label{eq:grad_x_0}
    \\&\nabla_{\lambda} \mathcal{L}(\mathbf{x}^*,\lambda^*) 
    = \langle \mathbf{x}^*, \mathbf{1} \rangle - v(\mathbf{1}) + v(\mathbf{0}) = 0.
    \label{eq:grad_lambda_0}
\end{align}
Together, Equations \ref{eq:grad_x_0} and \ref{eq:grad_lambda_0} tell us that
\begin{align}
   \langle \mathbf{1},  {\mathbf{x}^*} \rangle = v(\mathbf{1}) - v(\mathbf{0})
   = \mathbf{1}^\top (\mathbf{Z^\top Z})^{-1} \left(\mathbf{Z^\top y} - \frac{\lambda^*}{2}\mathbf{1}\right)
\end{align}

Plugging back into Equation \ref{eq:grad_x_0}, we have that
\begin{align}\label{eq:regression_solution}
    \mathbf{x}^* = (\mathbf{Z^\top Z})^{-1} \left( \mathbf{Z^\top y} 
    - \frac{\mathbf{1}^\top (\mathbf{Z^\top Z})^{-1}\mathbf{Z}^\top \mathbf{y} - v(\mathbf{1}) + v(\mathbf{0})}
    {\mathbf{1}^\top (\mathbf{Z^\top Z})^{-1} \mathbf{1}} \mathbf{1} \right).
\end{align}

\end{document}

%% file: math_commands.tex
\usepackage{amsmath,amsfonts,bm}

\def\eqref#1{equation~\ref{#1}}
\def\Eqref#1{Equation~\ref{#1}}

\def\1{\bm{1}}

\DeclareMathAlphabet{\mathsfit}{\encodingdefault}{\sfdefault}{m}{sl}
\SetMathAlphabet{\mathsfit}{bold}{\encodingdefault}{\sfdefault}{bx}{n}

\newcommand{\E}{\mathbb{E}}

\newcommand{\R}{\mathbb{R}}

\DeclareMathOperator*{\argmin}{arg\,min}

%% file: bib_macros.tex
  \usepackage{nth}
  \usepackage{intcalc}

  \newcommand{\cAAAI}[1]{AAAI\ Conference\ on\ Artificial (AAAI)}

%% file: tables/main_shap_error.tex
\resizebox{\linewidth}{!}{ 
\begin{tabular} {lcccccccc}
\toprule
 & \textbf{IRIS} & \textbf{California} & \textbf{Diabetes} & \textbf{Adult} & \textbf{Correlated} & \textbf{Independent} & \textbf{NHANES} & \textbf{Communities} \\ 
\midrule
\addlinespace[1ex] 
\textbf{Kernel SHAP} &  &  &  &  &  &  &  &  \\ 
\hspace{7pt}Mean & \cellcolor{bronze!60}0.0326 & \cellcolor{bronze!60}0.03 & \cellcolor{bronze!60}0.0592 & \cellcolor{bronze!60}0.0596 & \cellcolor{bronze!60}0.0471 & \cellcolor{bronze!60}0.0269 & \cellcolor{bronze!60}0.0542 & \cellcolor{bronze!60}0.118 \\ 
\hspace{7pt}1st Quartile & \cellcolor{bronze!60}1.03e-05 & \cellcolor{bronze!60}0.00689 & \cellcolor{bronze!60}0.0123 & \cellcolor{bronze!60}0.0196 & \cellcolor{bronze!60}0.0264 & \cellcolor{bronze!60}0.0142 & \cellcolor{bronze!60}0.0175 & \cellcolor{bronze!60}0.0578 \\ 
\hspace{7pt}2nd Quartile & \cellcolor{bronze!60}0.000259 & \cellcolor{bronze!60}0.0177 & \cellcolor{bronze!60}0.0237 & \cellcolor{bronze!60}0.0412 & \cellcolor{bronze!60}0.0405 & \cellcolor{bronze!60}0.0207 & \cellcolor{bronze!60}0.038 & \cellcolor{bronze!60}0.0876 \\ 
\hspace{7pt}3rd Quartile & \cellcolor{bronze!60}0.034 & \cellcolor{bronze!60}0.0417 & \cellcolor{bronze!60}0.0545 & \cellcolor{bronze!60}0.0844 & \cellcolor{bronze!60}0.0556 & \cellcolor{bronze!60}0.0315 & \cellcolor{bronze!60}0.0789 & \cellcolor{bronze!60}0.139 \\ 
\addlinespace[1ex] 
\textbf{Optimized Kernel SHAP} &  &  &  &  &  &  &  &  \\ 
\hspace{7pt}Mean & \cellcolor{silver!60}4.84e-09 & \cellcolor{silver!60}0.00301 & \cellcolor{silver!60}0.00791 & \cellcolor{silver!60}0.0104 & \cellcolor{silver!60}0.0117 & \cellcolor{silver!60}0.00507 & \cellcolor{silver!60}0.00713 & \cellcolor{silver!60}0.0231 \\ 
\hspace{7pt}1st Quartile & \cellcolor{gold!60}1.66e-13 & \cellcolor{silver!60}0.000805 & \cellcolor{silver!60}0.00147 & \cellcolor{silver!60}0.00309 & \cellcolor{silver!60}0.00432 & \cellcolor{silver!60}0.00214 & \cellcolor{silver!60}0.00151 & \cellcolor{silver!60}0.00879 \\ 
\hspace{7pt}2nd Quartile & \cellcolor{silver!60}2.17e-13 & \cellcolor{silver!60}0.00205 & \cellcolor{silver!60}0.00367 & \cellcolor{silver!60}0.00491 & \cellcolor{silver!60}0.00848 & \cellcolor{silver!60}0.00459 & \cellcolor{silver!60}0.00408 & \cellcolor{silver!60}0.0161 \\ 
\hspace{7pt}3rd Quartile & \cellcolor{silver!60}2.69e-10 & \cellcolor{silver!60}0.00391 & \cellcolor{silver!60}0.00806 & \cellcolor{silver!60}0.0109 & \cellcolor{silver!60}0.0159 & \cellcolor{silver!60}0.00691 & \cellcolor{silver!60}0.00827 & \cellcolor{silver!60}0.0301 \\ 
\addlinespace[1ex] 
\textbf{Leverage SHAP} &  &  &  &  &  &  &  &  \\ 
\hspace{7pt}Mean & \cellcolor{gold!60}4.84e-09 & \cellcolor{gold!60}0.000249 & \cellcolor{gold!60}0.00227 & \cellcolor{gold!60}0.00416 & \cellcolor{gold!60}0.00662 & \cellcolor{gold!60}0.00257 & \cellcolor{gold!60}0.00505 & \cellcolor{gold!60}0.0139 \\ 
\hspace{7pt}1st Quartile & \cellcolor{silver!60}1.66e-13 & \cellcolor{gold!60}4.23e-05 & \cellcolor{gold!60}0.000202 & \cellcolor{gold!60}0.000424 & \cellcolor{gold!60}0.00223 & \cellcolor{gold!60}0.000936 & \cellcolor{gold!60}0.000816 & \cellcolor{gold!60}0.00575 \\ 
\hspace{7pt}2nd Quartile & \cellcolor{gold!60}2.17e-13 & \cellcolor{gold!60}0.000119 & \cellcolor{gold!60}0.000991 & \cellcolor{gold!60}0.00102 & \cellcolor{gold!60}0.00484 & \cellcolor{gold!60}0.00233 & \cellcolor{gold!60}0.00232 & \cellcolor{gold!60}0.00916 \\ 
\hspace{7pt}3rd Quartile & \cellcolor{gold!60}2.69e-10 & \cellcolor{gold!60}0.000277 & \cellcolor{gold!60}0.00259 & \cellcolor{gold!60}0.00362 & \cellcolor{gold!60}0.00865 & \cellcolor{gold!60}0.00378 & \cellcolor{gold!60}0.00407 & \cellcolor{gold!60}0.0213 \\ 
\bottomrule
\end{tabular}}

%% file: tables/error_ratio.tex
\begin{tabular}{lcccccc}
\toprule
Dataset & $5n$ & $10n$ & $20n$ & $40n$ & $80n$ & $160n$ \\
\midrule
IRIS & 1.00 & 1.00 & 1.00 & 1.00 & 1.00 & 1.00 \\
California & 0.0894 & 0.0824 & 0.0471 & 1.00 & 1.00 & 1.00 \\
Diabetes & 0.184 & 0.287 & 0.261 & 0.181 & 0.151 & 1.00 \\
Adult & 0.527 & 0.400 & 0.328 & 0.439 & 0.345 & 0.211 \\
Correlated & 0.588 & 0.568 & 0.660 & 0.657 & 0.727 & 0.674 \\
Independent & 0.492 & 0.507 & 0.625 & 0.654 & 0.618 & 0.651 \\
NHANES & 0.706 & 0.708 & 0.685 & 0.736 & 0.753 & 0.720 \\
Communities & 0.531 & 0.601 & 0.598 & 0.659 & 0.671 & 0.697 \\
\bottomrule
\end{tabular}

%% file: tables/gamma_distribution.tex
\begin{tabular}{lcccc}
    \toprule
    Dataset & $n$ & 1st Quartile & 2nd Quartile & 3rd Quartile \\
    \midrule
    IRIS & 4 & 0.000234 & 0.142 & 0.991 \\
    California & 8 & 0.161 & 0.297 & 0.44 \\
    Diabetes & 10 & 0.182 & 0.321 & 0.528 \\
    Adult & 12 & 0.171 & 0.4 & 0.729 \\
    \bottomrule
\end{tabular}

%% file: tables/ablation_shap_error.tex
\resizebox{\linewidth}{!}{
\begin{tabular} {lcccccccc}
\toprule
 & \textbf{IRIS} & \textbf{California} & \textbf{Diabetes} & \textbf{Adult} & \textbf{Correlated} & \textbf{Independent} & \textbf{NHANES} & \textbf{Communities} \\
\midrule
\addlinespace[1ex]
\textbf{Kernel SHAP} &  &  &  &  &  &  &  &  \\
\hspace{7pt}Mean & 0.0326 & 0.03 & 0.0592 & 0.0596 & 0.0471 & 0.0269 & 0.0542 & 0.118 \\
\hspace{7pt}1st Quartile & 1.03e-05 & 0.00689 & 0.0123 & 0.0196 & 0.0264 & 0.0142 & 0.0175 & 0.0578 \\
\hspace{7pt}2nd Quartile & 0.000259 & 0.0177 & 0.0237 & 0.0412 & 0.0405 & 0.0207 & 0.038 & 0.0876 \\
\hspace{7pt}3rd Quartile & 0.034 & 0.0417 & 0.0545 & 0.0844 & 0.0556 & 0.0315 & 0.0789 & 0.139 \\
\addlinespace[1ex]
\textbf{Kernel SHAP Paired} &  &  &  &  &  &  &  &  \\
\hspace{7pt}Mean & 0.00185 & 0.000688 & 0.00376 & 0.00635 & 0.0105 & 0.00425 & 0.00772 & 0.0257 \\
\hspace{7pt}1st Quartile & 1.86e-10 & \cellcolor{bronze!60}8.65e-05 & 0.000417 & 0.000748 & 0.00431 & 0.00141 & 0.00151 & 0.0108 \\
\hspace{7pt}2nd Quartile & 1.29e-07 & 0.000353 & 0.00192 & 0.00163 & 0.00775 & 0.0037 & 0.00414 & 0.0175 \\
\hspace{7pt}3rd Quartile & 0.000166 & 0.000877 & 0.00563 & 0.00451 & 0.0137 & 0.00654 & 0.00789 & 0.0361 \\
\addlinespace[1ex]
\textbf{Optimized Kernel SHAP} &  &  &  &  &  &  &  &  \\
\hspace{7pt}Mean & \cellcolor{bronze!60}4.84e-09 & 0.00301 & 0.00791 & 0.0104 & 0.0117 & 0.00507 & 0.00713 & 0.0231 \\
\hspace{7pt}1st Quartile & \cellcolor{gold!60}1.66e-13 & 0.000805 & 0.00147 & 0.00309 & 0.00432 & 0.00214 & 0.00151 & 0.00879 \\
\hspace{7pt}2nd Quartile & \cellcolor{silver!60}2.17e-13 & 0.00205 & 0.00367 & 0.00491 & 0.00848 & 0.00459 & 0.00408 & 0.0161 \\
\hspace{7pt}3rd Quartile & 2.69e-10 & 0.00391 & 0.00806 & 0.0109 & 0.0159 & 0.00691 & 0.00827 & 0.0301 \\
\addlinespace[1ex]
\textbf{Permutation SHAP} &  &  &  &  &  &  &  &  \\
\hspace{7pt}Mean & 0.00557 & 0.00131 & 0.00853 & 0.0132 & 0.0119 & 0.00529 & 0.00569 & 0.0291 \\
\hspace{7pt}1st Quartile & 1.03e-09 & 0.000175 & 0.000689 & 0.00143 & 0.00306 & 0.00123 & \cellcolor{gold!60}0.000652 & 0.0109 \\
\hspace{7pt}2nd Quartile & 1.26e-07 & 0.000658 & 0.00325 & 0.00288 & 0.00762 & 0.00324 & 0.00284 & 0.017 \\
\hspace{7pt}3rd Quartile & 0.000617 & 0.00214 & 0.00869 & 0.00619 & 0.0171 & 0.007 & 0.00639 & 0.0381 \\
\addlinespace[1ex]
\textbf{Leverage SHAP} &  &  &  &  &  &  &  &  \\
\hspace{7pt}Mean & \cellcolor{gold!60}4.84e-09 & \cellcolor{gold!60}0.000249 & \cellcolor{gold!60}0.00227 & \cellcolor{silver!60}0.00416 & \cellcolor{bronze!60}0.00662 & \cellcolor{gold!60}0.00257 & \cellcolor{bronze!60}0.00505 & \cellcolor{silver!60}0.0139 \\
\hspace{7pt}1st Quartile & \cellcolor{silver!60}1.66e-13 & \cellcolor{silver!60}4.23e-05 & \cellcolor{gold!60}0.000202 & \cellcolor{gold!60}0.000424 & \cellcolor{gold!60}0.00223 & \cellcolor{bronze!60}0.000936 & \cellcolor{silver!60}0.000816 & \cellcolor{gold!60}0.00575 \\
\hspace{7pt}2nd Quartile & \cellcolor{gold!60}2.17e-13 & \cellcolor{silver!60}0.000119 & \cellcolor{silver!60}0.000991 & \cellcolor{silver!60}0.00102 & \cellcolor{gold!60}0.00484 & \cellcolor{silver!60}0.00233 & \cellcolor{silver!60}0.00232 & \cellcolor{gold!60}0.00916 \\
\hspace{7pt}3rd Quartile & \cellcolor{silver!60}2.69e-10 & \cellcolor{gold!60}0.000277 & \cellcolor{gold!60}0.00259 & \cellcolor{silver!60}0.00362 & \cellcolor{silver!60}0.00865 & \cellcolor{gold!60}0.00378 & \cellcolor{gold!60}0.00407 & \cellcolor{bronze!60}0.0213 \\
\addlinespace[1ex]
\textbf{Leverage SHAP (Unpaired)} &  &  &  &  &  &  &  &  \\
\hspace{7pt}Mean & \cellcolor{silver!60}4.84e-09 & 0.0124 & 0.0328 & 0.0425 & 0.0308 & 0.0161 & 0.0332 & 0.0672 \\
\hspace{7pt}1st Quartile & 1.66e-13 & 0.0031 & 0.00576 & 0.0121 & 0.0178 & 0.0087 & 0.00887 & 0.0333 \\
\hspace{7pt}2nd Quartile & \cellcolor{bronze!60}2.17e-13 & 0.00775 & 0.0127 & 0.0274 & 0.0268 & 0.0125 & 0.0234 & 0.0509 \\
\hspace{7pt}3rd Quartile & \cellcolor{bronze!60}2.69e-10 & 0.0187 & 0.0256 & 0.0468 & 0.0351 & 0.0197 & 0.0485 & 0.0858 \\
\addlinespace[1ex]
\textbf{Leverage SHAP wo Bernoulli} &  &  &  &  &  &  &  &  \\
\hspace{7pt}Mean & 0.00136 & \cellcolor{bronze!60}0.000604 & \cellcolor{bronze!60}0.0036 & \cellcolor{bronze!60}0.00474 & \cellcolor{silver!60}0.00635 & \cellcolor{bronze!60}0.00263 & \cellcolor{silver!60}0.00465 & \cellcolor{gold!60}0.0136 \\
\hspace{7pt}1st Quartile & 3.07e-10 & 0.000108 & \cellcolor{bronze!60}0.000327 & \cellcolor{bronze!60}0.000584 & \cellcolor{bronze!60}0.00249 & \cellcolor{gold!60}0.000787 & 0.000851 & \cellcolor{silver!60}0.00575 \\
\hspace{7pt}2nd Quartile & 1.29e-07 & \cellcolor{bronze!60}0.000259 & \cellcolor{bronze!60}0.00159 & \cellcolor{bronze!60}0.00121 & \cellcolor{silver!60}0.00492 & \cellcolor{bronze!60}0.00241 & \cellcolor{bronze!60}0.00247 & \cellcolor{bronze!60}0.00996 \\
\hspace{7pt}3rd Quartile & 0.000185 & \cellcolor{bronze!60}0.000683 & \cellcolor{bronze!60}0.00499 & \cellcolor{bronze!60}0.00395 & \cellcolor{bronze!60}0.00898 & \cellcolor{silver!60}0.00386 & \cellcolor{silver!60}0.00437 & \cellcolor{gold!60}0.0187 \\
\addlinespace[1ex]
\textbf{Leverage SHAP (Binomial)} &  &  &  &  &  &  &  &  \\
\hspace{7pt}Mean & 4.84e-09 & \cellcolor{silver!60}0.000253 & \cellcolor{silver!60}0.00252 & \cellcolor{gold!60}0.00387 & \cellcolor{gold!60}0.00618 & \cellcolor{silver!60}0.00262 & \cellcolor{gold!60}0.00436 & \cellcolor{bronze!60}0.014 \\
\hspace{7pt}1st Quartile & \cellcolor{bronze!60}1.66e-13 & \cellcolor{gold!60}4.16e-05 & \cellcolor{silver!60}0.000222 & \cellcolor{silver!60}0.000445 & \cellcolor{silver!60}0.00233 & \cellcolor{silver!60}0.000812 & \cellcolor{bronze!60}0.000835 & \cellcolor{bronze!60}0.00608 \\
\hspace{7pt}2nd Quartile & 2.17e-13 & \cellcolor{gold!60}0.000114 & \cellcolor{gold!60}0.000921 & \cellcolor{gold!60}0.000969 & \cellcolor{bronze!60}0.00507 & \cellcolor{gold!60}0.00212 & \cellcolor{gold!60}0.00224 & \cellcolor{silver!60}0.00946 \\
\hspace{7pt}3rd Quartile & \cellcolor{gold!60}2.69e-10 & \cellcolor{silver!60}0.000329 & \cellcolor{silver!60}0.00287 & \cellcolor{gold!60}0.00284 & \cellcolor{gold!60}0.00822 & \cellcolor{bronze!60}0.00397 & \cellcolor{bronze!60}0.00442 & \cellcolor{silver!60}0.0199 \\
\bottomrule
\end{tabular}}